\documentclass[12pt]{article}

\usepackage{arxiv_package}
\usepackage{mathrsfs}
\usepackage{enumitem}
\usepackage{booktabs}
\usepackage{nicefrac}
\usepackage{bm}
\usepackage{amsthm}
\usepackage{amsmath}
\usepackage{amsfonts}
\usepackage{amssymb}
\usepackage{mathtools}
\usepackage{caption}
\usepackage{subcaption}
\usepackage{tabularx}
\usepackage{multirow}
\usepackage{makecell}
\usepackage{float}
\usepackage{appendix}
\usepackage{bbm}
\usepackage{dsfont}
\usepackage[normalem]{ulem}
\usepackage{cancel}
\usepackage{comment}
\usepackage{diagbox}

\usepackage{algorithm}
\usepackage{algpseudocode}
\graphicspath{{}}

\hypersetup{
    colorlinks=true,
    citecolor=blue,
    linkcolor=blue
}

\newcommand{\Pp}{\mathbb P}

\title{Conformalized Rate-Adaptive Sensing}

\author{
Jiawei Yang \quad Yao Zhang\\
\vspace{5pt}
Department of Statistics and Data Science\\
National University of Singapore
}

\begin{document}
\maketitle

\begin{abstract}
Many high-resolution imaging systems face the same fundamental question: when have enough measurements been collected to reconstruct an image accurately? We develop \emph{Conformalized Rate-Adaptive Sensing} (CoRAS), a method that adaptively chooses an acquisition or compression rate for each image while keeping the reconstruction error below a target level with high probability. As measurements are collected, an image reconstruction model gradually recovers the true image, producing a reconstruction path over acquisition rates.
CoRAS uses this path up to an early decision time to estimate the target stopping time, defined as the first time at which the reconstruction error falls below the target level. It then calibrates this estimate using images with similar early reconstruction behavior, producing an upper bound on the stopping time with marginal and approximate conditional coverage guarantees. Experiments on image datasets show that CoRAS attains the target stopping-time coverage, uses fewer measurements on average than fixed-rate stopping rules, and assigns more measurements to images that are harder to reconstruct.
\end{abstract}

\section{Introduction}
\label{sec:introduction}

Modern data systems often rely on compression and partial acquisition when raw high-dimensional data is too costly or too slow to acquire, transmit, or process in full. These constraints are common in sensing and imaging systems, where devices must operate under tight budgets of bits, energy, acquisition time, and latency. Yet in many applications, reconstruction quality cannot be sacrificed for efficiency: in medical imaging, for example, reducing the number of measurements can shorten acquisition time, but the reconstructed image must still allow radiologists to detect clinically relevant abnormalities \citep{liu2017current,flint2012optimal}. This trade-off motivates methods that improve efficiency while controlling reconstruction error.

Over the last decade, deep learning has become a central tool for improving image reconstruction quality in modern applications. Image reconstruction models \citep{ronneberger2015u,wang2020deep,reader2020deep} can recover high-quality images from compressed, corrupted, or partially observed inputs. However, the performance of these models can still vary substantially across images, especially when complex images are heavily compressed or undersampled. Recently, \citet{angelopoulos2022image} developed a conformal prediction method for constructing prediction intervals for reconstructed images, with distribution-free coverage guarantees for the true pixel values. These intervals can be viewed as post-hoc uncertainty estimates for reconstruction errors, and they raise a further question: how much information should be acquired to reduce this uncertainty enough for a good reconstruction?

More broadly, whenever collecting more information can improve quality, other applications face a similar \emph{stopping problem}: a system should acquire enough information to meet a target quality requirement, but no more than necessary, to limit cost. For example, a magnetic resonance imaging (MRI) scanner can collect more measurements to improve image quality \citep{pineda2020active}, a time-series classifier can wait for more observations before making a decision \citep{achenchabe2021early}, and a reasoning language model can spend more computation before returning an answer \citep{xie2026statistical}. In each case, waiting longer may improve quality, but past some point the extra measurements or computation only add cost without improving the final output. Where that point lies differs from one case to the next.

\subsection{Overview}

This paper develops \emph{Conformalized Rate-Adaptive Sensing} (CoRAS), a method for solving this stopping problem in image reconstruction. Here, the target stopping time is the first point along the sampling path at which the reconstruction error falls below a target level. This quantity is image-specific: an image with simple content may be reconstructed accurately from a heavily compressed observation, while a complex image with many details may require a higher sampling rate. It is also unobserved at test time, since the true image is not available when the acquisition decision is made.

In CoRAS, we track the sequence of reconstructed images produced as the sampling rate increases. Although the error itself is not observed, the way the reconstruction changes along this path tells us how much more sampling is needed. We use this reconstruction path to predict the image-specific stopping time in two steps. A \emph{horizontal} step extrapolates the early history of reconstruction residuals to estimate when the loss will drop below the target; the extrapolation is guided  by a theoretical model of the image generating process. A \emph{vertical} step then corrects the horizontal prediction using calibration images with similar states, where the state combines the horizontal prediction with the entropy of the current reconstruction as a measure of image complexity. Applying full conformal prediction to the corrected estimate yields a stopping time with finite-sample marginal validity for error control. We further develop a model-based theory showing when the horizontal and vertical steps together achieve approximate conditional validity across images of different complexity. Experiments on Fashion-MNIST and brain MRI data show that CoRAS attains the target stopping-time coverage while using fewer measurements on average than fixed-rate conformal baselines, and that it assigns more measurements to
harder images.

The rest of the paper is organized as follows. \Cref{sec:related} reviews related work. \Cref{sec:problem} sets up the image reconstruction problem and the fixed-rate stopping rules. \Cref{sec:method} introduces the proposed CoRAS framework. \Cref{sec:experiments} reports experiments on Fashion-MNIST and brain MRI data, and \Cref{sec:discussion} concludes with a discussion.

\section{Related work}
\label{sec:related}
Conformal prediction is a popular framework for turning the predictions of
complex machine learning models into prediction sets with finite-sample
marginal coverage under exchangeability
\citep{vovk2005algorithmic,Angelopoulos2024TheoreticalFO}. Recent methods seek stronger local or
approximate conditional guarantees by matching a test point to calibration
points with similar covariates \citep{guan2023localized,hore2025conformal},
learned representations \citep{kiyani2024conformal}, or learned cluster
memberships \citep{zhang2024posterior}. CoRAS differs from these methods in both
the information available at test time and the object being predicted. In our
setting, no baseline covariates are available, and the true image has not yet
been fully observed. CoRAS instead uses the early part of the reconstruction
path to predict the target stopping time and to construct a valid upper bound
for it. Both the path-based prediction and its coverage theory are new to the
conformal prediction literature.

A related line of work calibrates a predictive system to control a
user-specified loss. Risk-controlling prediction sets \citep{bates2021distribution}
and the Learn-then-Test framework \citep{angelopoulos2021learn} give
finite-sample control of a population risk, while conformal risk control
\citep{angelopoulos2024conformal} extends conformal prediction to any monotone
loss. These methods select a global tuning parameter from a
nested family of prediction sets, and the resulting guarantee is about a
population-level risk. CoRAS instead targets the stopping time along a
reconstruction path, and its output is an image-specific stopping time, or
equivalently an acquisition rate, rather than a prediction set. The guarantees
are also built differently: risk-controlling prediction sets apply a
concentration bound to an empirical risk estimate, and Learn-then-Test casts
risk control as a testing problem. CoRAS is closest to conformal risk control,
since both use exchangeability for finite-sample validity, but it calibrates an
image-specific stopping rule rather than a single shared parameter.

CoRAS is also related to synthetic control methods from causal inference.
Classical synthetic control estimates a treated unit's counterfactual as a
weighted average of control units matched on pre-treatment outcomes
\citep{abadie2003economic,abadie2010synthetic,abadie2021using}, while augmented
variants combine this matching with an outcome regression
\citep{ben2021augmented,arkhangelsky2021synthetic}. In this analogy, the
horizontal prediction in CoRAS plays the role of the outcome regression and the
vertical correction plays the role of the matching step. Two features set CoRAS
apart. Its horizontal model is tailored to image acquisition and is motivated by
a theoretical model of the image-generating process, and its horizontal and
vertical steps are embedded in a conformal framework, yielding finite-sample
marginal validity together with a new model-based theory for approximate
conditional coverage.

Adaptive sensing and early-decision problems arise in several application areas.
In accelerated MRI, active-acquisition methods learn image-dependent policies
for choosing which \(k\)-space measurements to collect, typically with
reinforcement learning or policy-gradient methods
\citep{pineda2020active,bakker2020experimental}. Early classification of time
series similarly trades off accuracy against the cost of waiting for more
observations \citep{achenchabe2021early}, and test-time scaling in large
language models spends more computation on harder prompts
\citep{snell2025scaling}. These methods mostly optimize an empirical trade-off
between accuracy and resource use. In comparison, CoRAS estimates an image-specific
stopping time and calibrates it into a valid stopping rule with a finite-sample
error guarantee.

\section{Calibration of stopping times}
\label{sec:problem}

\subsection{Image reconstruction}\label{sect:setup}

In many imaging systems, we want to represent an image using less information while still being able to reconstruct it accurately. This reduction can happen after the image is observed, as in standard compression, or during acquisition, where the system collects only part of the measurements needed to form the image. Our focus is the acquisition case, where the rate must be chosen before the full image is available.

Let \(Y\in\mathbb R^K\) denote the true image, and for each sampling rate \(\theta\in(0,1]\), let \(X(\theta)\) denote the observation available at rate \(\theta\). Throughout, \(X\) is what the system observes and \(Y\) is what it must recover. Smaller \(\theta\) retains less information, while \(\theta=1\) corresponds to the full observation. The reconstruction task is to recover \(Y\) from its low-rate observation \(X(\theta)\): lower rates save resources but make reconstruction harder.

A concrete example is MRI reconstruction from incomplete measurements collected in Fourier space, often called \emph{k}-space. To reduce acquisition time, one may collect only a subset of Fourier coefficients. Let \(\mathcal F\) denote the discrete Fourier transform, and let \(M_\theta\in\{0,1\}^K\) be a binary mask satisfying
\(
K^{-1}\sum_{j=1}^K (M_\theta)_j=\theta.
\)
A simple low-rate observation can then be formed by zero-filling the unmeasured Fourier coefficients:
\[
X(\theta)
=
\mathcal F^{-1}\!\left(M_\theta\odot \mathcal F(Y)\right).
\]
This observation is usually degraded by artifacts caused by the missing frequency information. The goal is to reconstruct \(Y\) by removing these artifacts.

\subsection{Raw stopping time}

We now ask how far along the sampling path a system must go before the error is small enough. We start with the simplest case, where no reconstruction model is used and the rate is calibrated directly from past images.

Let \(\{Y_i\}_{i=1}^n\) be a calibration dataset used to choose the sampling rate, and let
\[
0<\theta_1<\theta_2<\cdots<\theta_{t_{\max}}\leq 1
\]
be a fixed grid of sampling rates. Throughout this article, we assume that \(X_i(\theta_{t_{\max}})\) is close enough to \(Y_i\) that stopping at \(t_{\max}\) always controls the loss at the target level. 

At every time \(t\), we observe \(X_i(\theta_t)\) and the raw reconstruction loss
\begin{equation}\label{equ:lraw}
L_i^{\text{raw}}(t)
=
\frac1K\sum_{k=1}^K
\bigl[Y_{i,k}-X_{i,k}(\theta_t)\bigr]^2,
\end{equation}
which compares the observation with the true image, without using any reconstruction model. Given a target error level \(c>0\), we define the target stopping time as
\begin{equation}
\label{eq:target_time}
T_i^{\text{raw}}
:=
\min\{t\in[t_{\max}]:L_i^{\text{raw}}(t)\le c\}.
\end{equation}
This is the first time point at which the reconstruction loss drops below the target level. We observe \(T_i^{\text{raw}}\) for every calibration image \(Y_i\), \(i\in[n]\), and want to use these observed stopping times to construct an upper bound, that is, a stopping rule, for the unobserved stopping time of a test image \(Y_{n+1}\).

Given a miscoverage level \(\alpha\in(0,1)\), the raw stopping rule is defined as
\begin{equation}\label{equ:vraw}
V_{\alpha}^{\text{raw}} = Q_{1-\alpha}\!\left\{
\frac{1}{n+1}\sum_{i=1}^n\delta_{T_i^{\text{raw}}}
+\frac{1}{n+1}\delta_{t_{\max}}
\right\},
\end{equation}
where \(Q_{1-\alpha}\) denotes the \((1-\alpha)\)-quantile of the displayed discrete distribution, that is, the smallest time \(t\) at which the distribution accumulates at least mass \(1-\alpha\).

\begin{theorem}[Marginal validity]
\label{thm:rawmarginal}
Assume that the images \(Y_1,\ldots,Y_{n+1}\) are exchangeable, and that the loss in \eqref{equ:lraw} is nonincreasing in \(t\) with probability one. Then the stopping rule \(V_{\alpha}^{\text{raw}}\) defined in \eqref{equ:vraw} satisfies
\[
\Pp\{L_{n+1}^{\text{raw}}(V_{\alpha}^{\text{raw}})\le c\}\ge 1-\alpha.
\]
\end{theorem}

By the standard marginal validity proof of conformal prediction \citep{vovk2005algorithmic,Angelopoulos2024TheoreticalFO}, the exchangeability assumption implies that \(T_{n+1}^{\text{raw}}\le V_{\alpha}^{\text{raw}}\) with probability at least \(1-\alpha\). The monotonicity assumption ensures that
\[
T_{n+1}^{\text{raw}}\le t \quad \Longrightarrow\quad L_{n+1}^{\text{raw}}(t)\le L_{n+1}^{\text{raw}}(T_{n+1}^{\text{raw}})\le c,
\]
so \(V_{\alpha}^{\text{raw}}\) controls the loss \(L_{n+1}^{\text{raw}}(V_{\alpha}^{\text{raw}})\). Monotonicity is natural here because \(X_i(\theta_t)\), based on more measurements,
approximates \(Y_i\) more closely as \(t\) increases.

We note that the stopping rule \(V_{\alpha}^{\text{raw}}\) is the same for every test image, regardless of its complexity. Unlike the usual conformal prediction setting, here we have no baseline covariates to adapt the rule, so a single fixed rate must serve all test images. Although \Cref{thm:rawmarginal} guarantees marginal validity, this single rate tends to sample more than necessary on simple images, while leaving the error above the target on complex ones.

\subsection{Model-based stopping time}

The raw loss compares the observation with the true image directly.
Given a low-rate observation \(X(\theta)\), an image reconstruction model $\hat\mu$ can produce an estimate of the true image \(Y\). Denote the reconstruction by
\[
\hat Y(\theta)=\hat\mu(X(\theta)),
\]
and write \(\hat Y_k(\theta)\) for its \(k\)-th pixel. The model \(\hat\mu\) can be fitted with a squared reconstruction loss on separate training data \(\{(X_j(\theta),Y_j)\}_{j=1}^{n_{\text{train}}}\), where \(\theta\) varies across training images, so that \(\hat\mu\) can reconstruct images from observations at any rate.

 Throughout, \(\hat Y_i(t):=\hat Y_i(\theta_t)\) denotes the reconstruction from the first \(t\)
acquisition steps.
For every calibration image \(i\), we observe a sequence of reconstructed images
\[
\hat Y_i(1),\hat Y_i(2),\ldots,\hat Y_i(t_{\max}).
\]
As with the raw loss \(L_i^{\text{raw}}(t)\) above, we define the reconstruction loss of \(\hat Y_i(t)\) by
\begin{equation}\label{equ:lit}
L_i(t)
=
\frac1K\sum_{k=1}^K
\bigl[Y_{i,k}-\hat Y_{i,k}(t)\bigr]^2.
\end{equation}
We assume that the target level is attainable on the grid, so that \(L_i(t)\le c\) for some \(t\le t_{\max}\) for every image. Denote the target stopping time of image \(i\) by
\begin{equation}\label{equ:t_target}
T_{i}
:=
\min\{t\in[t_{\max}]:L_{i}(t)\le c\}.
\end{equation}
The fixed-rate stopping rule is the \((1-\alpha)\)-quantile of the calibration stopping times:
\begin{equation}\label{equ:fixed_rate}
\hat V_{\alpha}
:=
Q_{1-\alpha}\!\left\{
\frac{1}{n+1}\sum_{i=1}^n\delta_{T_i}
+\frac{1}{n+1}\delta_{t_{\max}}
\right\}.
\end{equation}
As in \Cref{thm:rawmarginal}, the exchangeability and loss monotonicity assumptions give
\[
\Pp\{L_{n+1}(\hat V_{\alpha})\le c\}\ge 1-\alpha.
\]
The rule \(\hat V_{\alpha}\) is again the same for every test image. Although it has marginal validity, it may still cover the stopping times of complex images with probability below \(1-\alpha\). The goal of the adaptive method below is to retain the same error control while allowing the selected rate to vary across test images.

\section{CoRAS: Conformalized Rate-Adaptive Sensing}
\label{sec:method}

Standard conformal prediction methods produce prediction sets for responses  based on 
the covariates of the test point.  For the target stopping time $T_{n+1}$, we instead observe the
reconstruction path, and it lets us stop at the point where the predicted loss first falls
below $c$. CoRAS begins by acquiring every image up to an early decision time
\(t_0\), fixed in advance and the same for all images. In our experiments \(t_0\) is chosen so
that most images still have reconstruction error above the target level $c$, so
little acquisition is wasted. We then read off the path up to \(t_0\) and decide whether to
keep going and for how long. The rest of this section develops this idea.

\subsection{Horizontal prediction of stopping times}\label{subsec:horizontal}

\subsubsection{Next-band residual history}
We first develop the horizontal prediction model for acquisition schemes in
which each step directly reveals a new collection of coefficients of an
orthonormal transform of the target image.
Let \(\mathcal T\) be an orthonormal transform in which acquisition is ordered: a
normalized Fourier transform in a Fourier acquisition model, or an orthonormal wavelet
transform in a multiresolution one. For the horizontal model, we additionally assume
that the final step acquires the full observation, \(\theta_{t_{\max}}=1\), so that the
acquisition exhausts the transform coordinates. Denote the nested orthogonal
projections in the transform domain by
\[
P_{\le 1}\preceq P_{\le 2}\preceq\cdots\preceq P_{\le t_{\max}}=\mathrm{Id},
\]
where \(P_{\le t}\) keeps the coordinates available after \(t\) acquisition steps, and define
\[
\Delta P_1:=P_{\le 1},
\qquad
\Delta P_t:=P_{\le t}-P_{\le t-1},
\qquad t=2,\ldots,t_{\max},
\]
so that \(\Delta P_t\) isolates the information added between \(t-1\) and \(t\). 
For nested orthogonal projections, the increments are mutually orthogonal and sum to the identity,
\[
\mathrm{Id}=\sum_{t=1}^{t_{\max}}\Delta P_t,
\qquad
\Delta P_t\Delta P_s=0
\quad\text{for }t\neq s .
\]
This band structure lets us write the squared reconstruction loss \(L_i(t)\) in \eqref{equ:lit} as a sum of
bandwise residual energies. Two facts do the work. Since \(\mathcal T\) is orthonormal, it
leaves squared lengths unchanged, \(\|\mathcal T v\|_2^2=\|v\|_2^2\); and once a vector is
split into mutually orthogonal pieces, its squared length is the sum of the squared lengths of
those pieces.
Writing \(v_i(t):=Y_i-\hat Y_i(t)\) for the reconstruction residual, with
\(\hat Y_i(t)\) shorthand for \(\hat Y_i(\theta_t)\), and applying the two facts in turn,
\begin{equation}
\label{eq:l2_band_decomposition_residual_state}
L_i(t)
=
K^{-1}\|v_i(t)\|_2^2
=
K^{-1}\|\mathcal T v_i(t)\|_2^2
=
K^{-1}\sum_{s=1}^{t_{\max}}\|\Delta P_s\mathcal T v_i(t)\|_2^2
=
\sum_{s=1}^{t_{\max}}
R_{i,s}(t),
\end{equation}
where
\(
R_{i,s}(t):=K^{-1}\|\Delta P_s\mathcal T v_i(t)\|_2^2
\)
is the contribution of band \(s\) to the reconstruction error at time \(t\), and the third equality is Parseval's identity for the bands.

For \(t\ge2\), define the one-step-ahead residual
\begin{equation}
\label{eq:next_band_prediction_residual}
I_{i,t}
:=
R_{i,t}(t-1)
=
K^{-1}
\left\|
\Delta P_t\mathcal T Y_i
-
\Delta P_t\mathcal T\hat Y_i(t-1)
\right\|_2^2 .
\end{equation}
At time \(t-1\), band \(t\) has not yet been acquired, so \(I_{i,t}\) measures how well
the reconstruction at time \(t-1\) predicted the new information revealed at step
\(t\). Once band \(t\) is acquired, the coefficients \(\Delta P_t\mathcal T Y_i\) are
measured, so \(I_{i,t}\) can be computed without knowing the full image. A large
\(I_{i,t}\) means the new band carries residual structure the previous reconstruction
missed; a small one means the band was well predicted.

This observability holds when the measurements are coefficients of
\(\mathcal T\): in single-channel Fourier acquisition, the newly acquired \(k\)-space
columns are exactly \(\Delta P_t\mathcal F(Y_i)\), so \(I_{i,t}\) compares measured
columns with the same columns of the reconstruction's transform. In multi-coil MRI,
these coefficients are not directly measured because the coil combination is
nonlinear; our M4Raw experiment therefore computes the residual retrospectively from the
fully sampled reference; details are given in Appendix~\ref{sect:M4Raw}.

At the decision time \(t_0\), we use the residual history \(I_{i,2},\ldots,I_{i,t_0}\) to
extrapolate the future one-step residuals \(I_{i,t_0+1},\ldots,I_{i,t_{\max}}\). Since these
residuals are expected to decay as more of the image is revealed, we model them by a smooth
decaying curve. The next subsection formalizes this through an ordered-tail model.

\subsubsection{Continuous ordered-tail model}
Let \(1=\omega_1<\omega_2<\cdots<\omega_{t_{\max}}\) be ordered frequency cutoffs, so that the
\(t\)-th newly acquired band is \((\omega_{t-1},\omega_t]\). We take the cutoffs to be
logarithmically spaced,
\begin{equation}
\label{equ:omegat}
\omega_t=\exp\{\Delta(t-1)\},
\qquad \Delta>0 .
\end{equation}
This logarithmically spaced grid is a version of the coarse-to-fine grids used in multiresolution image representations. It allocates acquisition steps evenly across relative frequency scales rather than absolute frequency increments.

Let \(E_i(\tau):=Y_i-\hat Y_i(\tau)\) be the reconstruction residual at acquisition time
\(\tau\). The method uses the discrete residual \(I_{i,t}\) of
\eqref{eq:next_band_prediction_residual}. The model below describes its continuous
idealization: the projection \(\Delta P_t\) restricts transform coefficients to the band
\((\omega_{t-1},\omega_t]\), and the continuous analogue of the next-band residual is
\[
I^\star_{i,t}
=
K^{-1}
\int_{\omega_{t-1}}^{\omega_t}
\bigl| [\mathcal T E_i(t-1)](x) \bigr|^2\,dx .
\]

\begin{assumption}
\label{ass:ordered_tail_preprojection_residual}
Let \(x\ge1\) index an idealized continuous frequency or scale coordinate associated
with \(\mathcal T\). Suppose the transform coefficient function of
\(E_i(\tau)\) satisfies
\begin{equation}
\label{equ:E_i}
[\mathcal T E_i(\tau)](x)
=
A_i x^{-p_i(\tau)},
\qquad A_i>0,\quad p_i(\tau)>1/2 .
\end{equation}
\end{assumption}

In \Cref{ass:ordered_tail_preprojection_residual}, the coefficient \(A_i\) captures the overall
residual scale for image \(i\), while \(p_i(\tau)\) captures the shape of the residual tail
across the ordered transform coordinate. Larger \(p_i(\tau)\) means a steeper tail and less
residual energy at finer scales. The dependence of \(p_i(\tau)\) on \(\tau\) allows the residual
spectrum to change as more information is acquired. In practice,  local texture, reconstruction artifacts, and other deviations in coefficient
magnitude are not represented exactly by this idealized model.

\begin{proposition}
\label{prop:next_band_residual_power_law_decay}
Suppose \Cref{ass:ordered_tail_preprojection_residual} holds and the cutoffs are logarithmically
spaced as in \eqref{equ:omegat}. If \(p_i(\tau)\) is constant, say \(p_i(\tau)\equiv p_i\), then,
for \(t=2,\ldots,t_{\max}\),
\[
\log I^\star_{i,t}
=
\log\left[
K^{-1}A_i^2
\frac{1-\exp\{-(2p_i-1)\Delta\}}{2p_i-1}
\right]
-
(2p_i-1)\Delta(t-2).
\]
\end{proposition}

\Cref{prop:next_band_residual_power_law_decay} is the stationary case. When \(p_i(\tau)\) is
constant, the log next-band residual is exactly linear in the acquisition step \(t\), with slope
\(-(2p_i-1)\Delta\). In this idealized stationary model, the log residuals lie exactly on a line, so a
degree-one fit recovers their trend without approximation error.

In Appendix~\ref{app:horizontal_extra},
\Cref{prop:preprojection_residual_polynomial_trend} generalizes
\Cref{prop:next_band_residual_power_law_decay} to nonstationary residual
spectra: if \(p_i(\tau)\) is smooth, \(\log I^\star_{i,t}\) admits a local
polynomial approximation of order \(q\) near the decision time, with remainder
controlled by the smoothness of \(p_i\); constant \(p_i\) makes the order-one
approximation exact and recovers the proposition above. We also show that if
the residual tail becomes no flatter from one step to the next, then
\(I^\star_{i,t}\) is decreasing. The monotone constraint in our fit encodes
this, while an unconstrained polynomial fit on a short history can turn upward past the
last observation, which is implausible when each step reveals more information of the
image. Finally, logarithmic spacing only motivates the horizontal extrapolation model: on
a general grid the exact linear relationship may not hold, but the low-order
monotone polynomial remains a flexible local approximation to the residual
trend.

\subsubsection{Image-wise monotone polynomial fit}
In the actual method, we observe the finite-dimensional residuals
\(I_{i,2},\ldots,I_{i,t_0}\), not the continuous idealization \(I^\star_{i,t}\). We
therefore work with
\begin{equation}\label{equ:def_H}
H_{i,t}:=\log(I_{i,t}+\varepsilon),
\qquad t=2,\ldots,t_0,
\end{equation}
where \(\varepsilon>0\) is a small numerical constant. Let
\[
\mathcal R_0:=\{2,\ldots,t_0\},
\qquad
\mathcal R_0^{\mathrm c}:=\{t_0+1,\ldots,t_{\max}\}.
\]
Let \(q_{\max}=5\). For each \(r\in\{1,\ldots,q_{\max}\}\), let
\(
\phi_r(t)
=
\bigl(\phi_{r,0}(t),\phi_{r,1}(t),\ldots,\phi_{r,r}(t)\bigr)^\top
\)
be a vector of basis functions, e.g., the centered polynomial basis
\(\phi_{r,\ell}(t)=(t-t_0)^\ell\). Denote the full order-\(r\) coefficient space by
\(\Gamma_r=\mathbb R^{r+1}\), and define the monotone class
\[
\Gamma_r^{\mathrm{mon}}
:=
\left\{
\gamma\in\Gamma_r:
\{\phi_r(t)-\phi_r(t+1)\}^\top\gamma\ge0,\ 
t=2,\ldots,t_{\max}-1
\right\}.
\]
The polynomial order controls the flexibility of the horizontal extrapolation: larger
orders can capture changing decay rates or mild curvature, but they also introduce more
coefficients to estimate from the short pre-decision history. For each image \(i\), we
therefore choose the order image-wise by cross-validation over its observed pre-decision
residual history. Let \((\mathcal A_b,\mathcal V_b)\), \(b=1,\ldots,B\), be folds with
\(\mathcal A_b\subset\mathcal R_0\) and \(\mathcal V_b=\mathcal R_0\setminus\mathcal A_b\).
We restrict the candidate order set to
\[
\mathcal Q
:=
\left\{
r\in\{1,\ldots,q_{\max}\}:
r+1\le \min_{1\le b\le B}|\mathcal A_b|
\right\},
\]
so that every training-fold polynomial design has full column rank. For each order
\(r\in\mathcal Q\) and fold \(b\), fit the monotone curve using only \(\mathcal A_b\):
\[
\hat\gamma_{i,r}^{(-b),\mathrm{mon}}
\in
\arg\min_{\gamma\in\Gamma_r^{\mathrm{mon}}}
\sum_{t\in\mathcal A_b}
\left\{
H_{i,t}-\phi_r(t)^\top\gamma
\right\}^2 .
\]
This is a small convex quadratic program with linear constraints. Only a few polynomial coefficients are optimized, and the problem can be solved
efficiently by a standard active-set or interior-point quadratic-programming solver.
We set \(\hat q_i\in\arg\min_{r\in\mathcal Q}\ell_{i,H}^{\mathrm{cv}}(r)\), choosing
the smallest minimizer if there is a tie, where
\[
\ell_{i,H}^{\mathrm{cv}}(r)
=
\frac1B
\sum_{b=1}^B
\frac1{|\mathcal V_b|}
\sum_{t\in\mathcal V_b}
\left\{
H_{i,t}-\phi_r(t)^\top\hat\gamma_{i,r}^{(-b),\mathrm{mon}}
\right\}^2
\]
is the image-wise cross-validation loss. Simple residual histories tend to select a
lower-order curve, while histories with more curvature select a larger order. For
calibration images, the fitted curves and validation losses are computed and stored once;
for a new test image, only that image's quantities are computed.
After selecting \(\hat q_i\), we solve the same monotone least-squares problem on the
full pre-decision index set \(\mathcal R_0\), and denote the resulting coefficient vector
by \(\hat\gamma_{i,\hat q_i}^{\mathrm{mon}}\).

When the selected order is large enough to contain the true trend, the
constrained fit satisfies a sharper deterministic extrapolation bound on
\(\mathcal R_0^{\mathrm c}\) than the unconstrained fit, and the improvement is the
part of the unconstrained fit removed by the projection, namely the local upward movements that are inconsistent with the dominant decaying trend;
\Cref{prop:monotone_projection_gain} in Appendix \ref{app:monotone_projection}
quantifies this benefit, and there we also discuss how the constraint eases the
bias--variance trade-off in selecting \(\hat q_i\).

Combining \eqref{eq:l2_band_decomposition_residual_state} and \eqref{equ:def_H}, we use the
extrapolated future band residuals as a horizontal proxy for the unresolved residual tail.
For \(t=t_0,\ldots,t_{\max}\), define
\[
\hat L_i^H(t)
:=
\sum_{s=t+1}^{t_{\max}}\hat I_{i,s}^H,
\qquad
\hat I_{i,s}^H
:=
\max\left\{
\exp\!\left(\phi_{\hat q_i}(s)^\top
\hat\gamma_{i,\hat q_i}^{\mathrm{mon}}\right)
-\varepsilon,
0
\right\}.
\]
This proxy treats the bands already acquired by time \(t\) as contributing little to the
loss, and approximates the residual left in band \(s>t\) by the value it takes just before
band \(s\) is acquired. The corresponding horizontal plug-in stopping time is
\begin{equation}\label{equ:T_H}
\hat T_i^H
=
\min\left\{
t\in\{t_0,\ldots,t_{\max}\}:
\hat L_i^H(t)\le c
\right\},
\end{equation}
with the convention \(\hat T_i^H=t_{\max}\) if the set is empty.

\subsection{Vertical calibration of the stopping time}
\label{subsec:vertical}
So far we have predicted the stopping time \(T_{n+1}\) by the horizontal prediction
\(\hat T_{n+1}^H\). Because only a short residual history is available at the decision time \(t_0\), the
horizontal predictions may be less accurate for images with complex structures. We
next use the observed stopping times \(T_1,\ldots,T_n\) of the calibration images to estimate and correct the
remaining bias of the horizontal prediction.

To measure image complexity, we compute the entropy of the
reconstruction as follows. Suppose that image intensities are scaled to \([0,1]\). Define
the 16 intensity bins
\(
B_b=[(b-1)/16,b/16),
b=1,\ldots,15,
\)
and
\(
B_{16}=[15/16,1].
\)
For the reconstructed image \(\hat Y_i(t_0)\), define its normalized 16-bin entropy by
\[
\hat\Omega_i
:=
\operatorname{Ent}_{16}\{\hat Y_i(t_0)\}
=
-\frac{1}{\log 16}
\sum_{b:q_b\{\hat Y_i(t_0)\}>0}
q_b\{\hat Y_i(t_0)\}
\log q_b\{\hat Y_i(t_0)\},
\]
where
\(
q_b\{\hat Y_i(t_0)\}
=
K^{-1}
\sum_{k\in[K]}
\mathbbm{1}\{\hat Y_{i,k}(t_0)\in B_b\}.
\)
In the experiments, we will evaluate coverage across images with different true-image
entropies \(\Omega_i=\operatorname{Ent}_{16}(Y_i)\). This true-image entropy is used only for
evaluation and is not available to the stopping rule.
A more task-specific measure of image complexity is the horizontal prediction itself.
A larger value indicates that the early residual trajectory calls for more acquisition.
We combine these quantities to define the state
\[
S_i
:=
\left(
\hat T_i^H,\hat\Omega_i
\right)^\top
\in\mathbb R^2.
\]
In the experiments, the horizontal coordinate is linearly rescaled to \([0,1]\)
before distances are computed; this deterministic rescaling is equivalent to a
corresponding rescaling of \(h_T\). Let \(g=(h_T,h_\Omega)\) denote a bandwidth pair, where \(h_T>0\) and \(h_\Omega>0\). We
define the bandwidth-scaled distance between images \(j\) and \(k\) by
\begin{equation}
\label{eq:anisotropic_vertical_distance}
\ell_{jk}(g)
:=
h_T^{-2}
\left(\hat T_j^H-\hat T_k^H\right)^2
+
h_\Omega^{-2}
\left(\hat\Omega_j-\hat\Omega_k\right)^2.
\end{equation}
A smaller \(h_T\) requires tighter matching in the horizontal prediction, whereas a
smaller \(h_\Omega\) requires tighter matching in decision-time reconstruction entropy.
For a fixed bandwidth pair \(g\), define the leave-one-out weights
\begin{equation}
\label{eq:kl_vertical_weights_closed_form}
\hat\omega_{j,k}(g)
=
\frac{
\exp\{-\ell_{jk}(g)\}
}{
\displaystyle
\sum_{l\in[n+1]\setminus\{j\}}
\exp\{-\ell_{jl}(g)\}
},
\qquad
k\in[n+1]\setminus\{j\}.
\end{equation}
Thus, image \(k\) receives greater weight \(\hat\omega_{j,k}(g)\) when its horizontal
prediction and decision-time entropy are closer to those of image \(j\).

Because the stopping rule is deployed only after observing the reconstruction path up
to \(t_0\), it cannot stop before \(t_0\). We therefore define the floored target
stopping time
\[
T_i^\star
:=
\max\{T_i,t_0\}.
\]
For each candidate value \(t\in\{t_0,\ldots,t_{\max}\}\), define the augmented floored
labels
\[
T_k^\star(t)
=
\begin{cases}
T_k^\star, & k=1,\ldots,n,\\
t, & k=n+1.
\end{cases}
\]
For exact conformal validity, the bandwidth-selection algorithm must treat the
calibration and test observations symmetrically. Let \(\mathcal G\) be a
prespecified finite grid of strictly positive bandwidth pairs, and let
\(\mathcal A\) denote the algorithm that maps the augmented sample
\(
\{(S_1,T_1^\star),\ldots,(S_n,T_n^\star),(S_{n+1},t)\}
\)
to a selected bandwidth pair \(\hat g(t)\in\mathcal G\). We require the algorithm
\(\mathcal A\) to be permutation invariant:
\begin{equation}
\label{equ:sym}
\mathcal A
\left(
(s_{\sigma(1)},t_{\sigma(1)}^\star(t)),
\ldots,
(s_{\sigma(n+1)},t_{\sigma(n+1)}^\star(t))
\right)
=
\mathcal A
\left(
(s_1,t_1^\star(t)),
\ldots,
(s_{n+1},t_{n+1}^\star(t))
\right)
\end{equation}
for any values \((s_1,t_1^\star(t)),\ldots,(s_{n+1},t_{n+1}^\star(t))\) of the \(n+1\)
observations and any permutation \(\sigma\) of \([n+1]\). 
For the remainder of the construction, write
\[
\hat\omega_{j,k}(t)
:=
\hat\omega_{j,k}\{\hat g(t)\}.
\]
Although the bandwidth pair is reselected for every candidate \(t\), the selection criterion is a squared-error loss in \(t-\hat T_{n+1}^H\), and
Appendix~\ref{app:bandwidth_selection} gives an online algorithm that updates the selection
across candidates efficiently. As a result, the per-test-image cost is linear in the calibration sample size for a fixed bandwidth grid \(\mathcal G\) and a fixed candidate range.
The same appendix also describes an approximation that selects the bandwidth pair once for all candidates \(t\) using only the calibration stopping times. Under the conditions stated there, the
calibration-only and exact procedures select the same bandwidths with probability tending to one. However, the calibration-only approximation is not directly covered by the validity result below. 

For each augmented point \(j\), define the leave-one-out horizontal--vertical prediction
\begin{equation}
\label{eq:hv_prediction}
\hat T_j^{HV}(t)
=
\hat T_j^H
+
\sum_{k\in[n+1]\setminus\{j\}}
\hat\omega_{j,k}(t)
\left\{
T_k^\star(t)-\hat T_k^H
\right\}.
\end{equation}
The first term, \(\hat T_j^H\), is the horizontal prediction for image \(j\). The
second term is a vertical correction that averages the horizontal prediction errors
\(T_k^\star(t)-\hat T_k^H\) among images with states close to \(S_j\). If the
horizontal model tends to predict stopping times that are too early for images similar
to \(j\), the correction is positive and moves \(\hat T_j^{HV}(t)\) later. If the
horizontal model tends to predict stopping times that are too late, the correction is
negative and moves the prediction earlier.

Define the corresponding residual score after the vertical correction:
\begin{equation}
\label{eq:hv_residual_score}
R_j(t)
:=
T_j^\star(t)-\hat T_j^{HV}(t)
=
e_j(t)
-
\sum_{k\ne j}
\hat\omega_{j,k}(t)e_k(t),
\end{equation}
where \(e_j(t):=T_j^\star(t)-\hat T_j^H\) is the horizontal residual under candidate
\(t\). This residualization is intended to improve local calibration. The raw horizontal
residual may contain systematic bias that varies across image complexities. The weighted
average \(\sum_{k\ne j}\hat\omega_{j,k}(t)e_k(t)\) estimates this bias using images with
similar states. 

If the correction removes most of the state-dependent bias and \(t\) is a typical value of \(T_{n+1}^\star\), the test residual
\(R_{n+1}(t)\) should be comparable with \(R_1(t),\ldots,R_n(t)\).
Conformal prediction retains the candidates \(t\) for which \(R_{n+1}(t)\) is not unusually
large. More precisely, the conformal \(p\)-value
for candidate \(t\) is given by 
\[
p^{HV}(t)
=
\frac{1}{n+1}
\sum_{j=1}^{n+1}
\mathbbm{1}
\left\{
R_j(t)\ge R_{n+1}(t)
\right\}.
\]
A small value of \(p^{HV}(t)\) means that \(R_{n+1}(t)\) is
unusually large relative to the other scores, so candidate \(t\) is rejected. If
\(p^{HV}(t)>\alpha\), then \(t\) is retained as a plausible value of \(T_{n+1}^\star\). The final conformalized stopping rule is
\begin{equation}
\label{equ:u_alpha}
\hat U_{n+1}^{HV}
=
\max\left\{
t\in\{t_0,\ldots,t_{\max}\}:
p^{HV}(t)>\alpha
\right\},
\end{equation}
with the convention that \(\hat U_{n+1}^{HV}=t_{\max}\) if the retained set is empty.
Taking the largest retained candidate guarantees that
\(\hat U_{n+1}^{HV}\ge T_{n+1}^\star\) whenever \(T_{n+1}^\star\) is retained, which is
what the coverage statement below requires.
\begin{theorem}[Marginal validity]
\label{thm:hv_marginal_validity}
Assume that the images \(Y_1,\ldots,Y_{n+1}\) are exchangeable, that the
bandwidth-selection algorithm satisfies the symmetry condition in \eqref{equ:sym},
and that the loss \(L_i(t)\) in \eqref{equ:lit} is nonincreasing in \(t\) with
probability one. Then the horizontal--vertical stopping rule
\(\hat U_{n+1}^{HV}\) in \eqref{equ:u_alpha} satisfies
\[
\Pp\left\{
L_{n+1}(\hat U_{n+1}^{HV})\le c
\right\}
\ge
1-\alpha.
\]
\end{theorem}
The theorem gives a marginal reconstruction-error guarantee without requiring the
horizontal extrapolation model to be correctly specified. We next show that our method has approximate conditional validity across images
with different states. Write \(e_j^\star:=T_j^\star-\hat T_j^H\) for the horizontal
residual at the true floored stopping time.

\begin{assumption}
\label{ass:feature_bias_model}
There is a bounded positive-definite kernel
\(\kappa:\mathcal S\times\mathcal S\to\mathbb R\) on the state space
\(\mathcal S\), whose associated reproducing kernel Hilbert space (RKHS)
\(\mathcal H_\kappa\) contains a function \(m_n\) such that
\[
e_j^\star
=
m_n(S_j)+\xi_j,
\qquad
j=1,\ldots,n+1,
\qquad
\|m_n\|_{\mathcal H_\kappa}
\le
\Lambda_{\kappa,n}.
\]
The error terms \(\xi_1,\ldots,\xi_{n+1}\) are i.i.d., mean-zero,
\(\sigma^2\)-sub-Gaussian, and independent of the states
\(S_1,\ldots,S_{n+1}\). They have a common density \(f_\xi\) satisfying
\(\|f_\xi\|_\infty\le M\).
\end{assumption}

Here, \(\|\cdot\|_{\mathcal H_\kappa}\) denotes the norm induced by the RKHS
inner product associated with \(\kappa\), which satisfies
\(
\langle
\kappa(s,\cdot),\kappa(s',\cdot)
\rangle_{\mathcal H_\kappa}
=
\kappa(s,s').
\)
The RKHS assumption is flexible: depending on the kernel, \(\mathcal H_\kappa\) may be a finite-dimensional feature space or an infinite-dimensional one, such as a Sobolev RKHS.
The function \(m_n(S_j)\) represents the part of the horizontal prediction error
that can be explained by the observed state, whereas \(\xi_j\) represents the
remaining idiosyncratic error. For example, two images with the same state may still differ in the fine-scale structure remaining in their unacquired bands, and \(\xi_j\) captures this unexplained variation.

The bounded-density condition is an idealized continuous approximation used to
simplify the exposition of the conditional-coverage theorem below. In the implementation, \(e_j^\star\) is grid-valued because both \(T_j^\star\) and \(\hat T_j^H\) are evaluated on a discrete acquisition grid. The proof of the theorem first establishes a more general
bound in terms of the concentration function of \(\xi_j\), which does not
require continuity. Under the bounded-density
condition, the probability that \(\xi_j\) lies in an interval of radius \(r\)
is at most \(2Mr\). For a discrete error distribution, the general bound reflects the
probability of exact ties, which need not vanish as the perturbation radius
approaches zero.

For the following result, fix a bandwidth pair
\(g=(h_T,h_\Omega)\) and suppress its dependence in the weights. For each \(j\in[n+1]\), let
\(\widehat P_j=\sum_{k\ne j}\hat\omega_{j,k}\,\delta_{S_k}\) be the weighted
neighborhood used to correct the horizontal prediction at \(S_j\). Define the
largest local maximum mean discrepancy (MMD) between a state and its
neighborhood,
\[
D_{\kappa,n}
:=
\max_{1\le j\le n+1}
\operatorname{MMD}_\kappa\bigl(\delta_{S_j},\widehat P_j\bigr)
=
\max_{1\le j\le n+1}
\Bigl\|
\kappa(S_j,\cdot)-\sum_{k\ne j}\hat\omega_{j,k}\,\kappa(S_k,\cdot)
\Bigr\|_{\mathcal H_\kappa}.
\]
By the dual representation of the MMD,
\[
D_{\kappa,n}
=
\max_{1\le j\le n+1}\;
\sup_{\|f\|_{\mathcal H_\kappa}\le1}
\Bigl|f(S_j)-\sum_{k\ne j}\hat\omega_{j,k}\,f(S_k)\Bigr|.
\]
Thus, \(D_{\kappa,n}\) measures how accurately each weighted neighborhood
reproduces evaluations of unit-norm RKHS functions at its target state.

\begin{theorem}[Approximate conditional validity]
\label{thm:hv_approx_conditional_coverage}
Suppose \Cref{ass:feature_bias_model} holds and that the loss \(L_i(t)\) in
\eqref{equ:lit} is nonincreasing in \(t\) with probability one. Let
\(\hat U_{n+1}^{HV}(g)\) denote the horizontal--vertical stopping rule in
\eqref{equ:u_alpha} using a fixed bandwidth pair
\(g=(h_T,h_\Omega)\). For \(n\ge2\), there exists a constant
\(\eta_g>0\), independent of \(n\), such that
\[
\Pp\left\{
L_{n+1}\{\hat U_{n+1}^{HV}(g)\}\le c
\,\middle|\,
S_{n+1}
\right\}
\ge
1-\alpha
-
4M\Lambda_{\kappa,n}
\mathbb E\left[
D_{\kappa,n}
\,\middle|\,
S_{n+1}
\right]
-
\eta_g
\sqrt{\frac{\log n}{n}}.
\]
\end{theorem}

\Cref{thm:hv_approx_conditional_coverage} bounds the conditional coverage gap
by the sum of a state-dependent component
\(
4M\Lambda_{\kappa,n}
\mathbb E[D_{\kappa,n}\mid S_{n+1}]
\)
and a stochastic component of order
\(\sqrt{\log(n)/n}\). The overall rate is governed by whichever component
vanishes more slowly.

In the state-dependent component, \(\Lambda_{\kappa,n}\) measures the
magnitude of the systematic error remaining after horizontal extrapolation,
whereas \(D_{\kappa,n}\) measures how well the vertical weights match the
states. Because \(\kappa\) is bounded, \(D_{\kappa,n}\) is uniformly bounded.
Consequently, the state-dependent component vanishes if
\(\Lambda_{\kappa,n}\to0\), even when the states are not matched perfectly.
Conversely, if \(\Lambda_{\kappa,n}=O(1)\), it vanishes whenever
\(
\mathbb E[D_{\kappa,n}\mid S_{n+1}]\to0.
\)
When both factors decrease, their rates multiply.

The stochastic component arises from two sources: the weighted noise averages
\(
\sum_{k\ne j}\hat\omega_{j,k}\xi_k
\)
and the empirical number of near ties in the conformal rank comparison.
For a
fixed positive bandwidth pair, boundedness of the state space implies that each
kernel weight is of order \(1/n\), so the weighted averages concentrate at the
usual \(n^{-1/2}\) rate. Controlling all \(n+1\) leave-one-out averages and the
near-tie count introduces the additional \(\sqrt{\log n}\) factor.
In particular, if
\(
\Lambda_{\kappa,n}
\mathbb E[D_{\kappa,n}\mid S_{n+1}]
\)
is of order \(\sqrt{\log(n)/n}\) or smaller, then the upper bound on the
conditional coverage gap also converges at the near-parametric rate
\(\sqrt{\log(n)/n}\).

In standard conformal prediction, conditional validity is stated given a test covariate
\(X_{n+1}\). Here the conditioning variable is the state
\(S_{n+1}=(\hat T_{n+1}^H,\hat\Omega_{n+1})^\top\), computed from the reconstruction path
up to the decision time \(t_0\), so the coverage guarantee conditions on exactly the information the stopping rule can use. A larger \(t_0\) may make the state more informative and the conditioning more
image-specific, at the cost of acquiring more of every image before deciding.

The symmetry condition in \eqref{equ:sym} gives marginal validity by
preserving exchangeability of the augmented residual scores.
\Cref{thm:hv_approx_conditional_coverage} asks for more: its proof treats the weights as
fixed when applying sub-Gaussian concentration to the residual noise, whereas a
candidate-specific bandwidth selected using the augmented stopping-time labels can depend
on the very noise being averaged, so the weighted concentration argument no longer applies. A bound uniform over the bandwidth grid would avoid this, but its MMD term would
be set by the worst bandwidth in the grid, which may never be selected in practice. We
therefore state the result for a fixed bandwidth pair.

\section{Experiments}
\label{sec:experiments}
We evaluate CoRAS on two datasets: Fashion-MNIST, a standard image benchmark, and M4Raw, a brain MRI dataset with raw \(k\)-space measurements. On both datasets, we compare three stopping rules: the raw fixed-rate rule in \eqref{equ:vraw}, calibrated using the raw loss in \eqref{equ:lraw}; the model fixed-rate rule in \eqref{equ:fixed_rate}; and the rate-adaptive rule in \eqref{equ:u_alpha}, calibrated from the loss in \eqref{equ:lit}.
The bandwidth in the adaptive rule is selected via the leave-one-out
error for predicting the stopping time; see Appendix \ref{app:bandwidth_selection} for more details.

We report three performance metrics for each method, writing \(\hat U_i\) generically
for its selected stopping time. The coverage rate is the fraction of test images with
\(\hat U_i \ge T_i\), which under the monotone loss assumption is equivalent to the
 reconstruction loss being at most \(c\). For CoRAS, this coincides with
\(\hat U_i \ge T_i^\star\), since \(\hat U_i^{HV}\ge t_0\) by construction. The
sampling rate is the mean selected rate \(\theta_{\hat U_i}\). The excess
sampling rate is the mean of \(\theta_{\hat U_i}-\theta_{T_i^\star}\), where the
floored target stopping time \(T_i^\star\) is computed from the loss in \eqref{equ:lit} for every
method, so that sampling efficiency is compared on a common scale. Code to reproduce the experiments is available at \url{https://github.com/jiawei-yang05/Conformalized_Rate_Adaptive_Sensing}.

\subsection{Fashion-MNIST}
\begin{figure}[h]
\centering
\begin{subfigure}{0.32\textwidth}
    \centering
    \includegraphics[width=\linewidth]{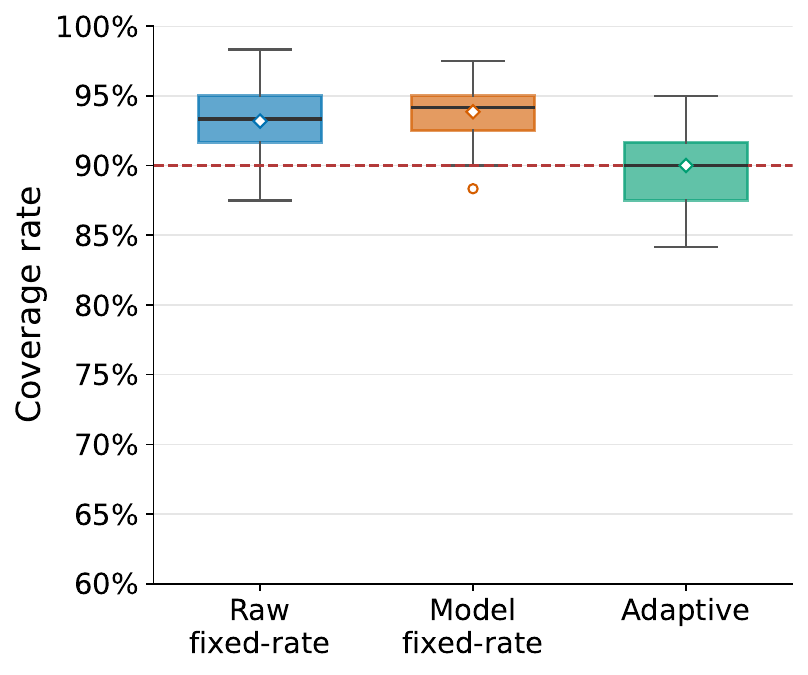}
    \caption{Coverage rate.}
\end{subfigure}
\hfill
\begin{subfigure}{0.32\textwidth}
    \centering
    \includegraphics[width=\linewidth]{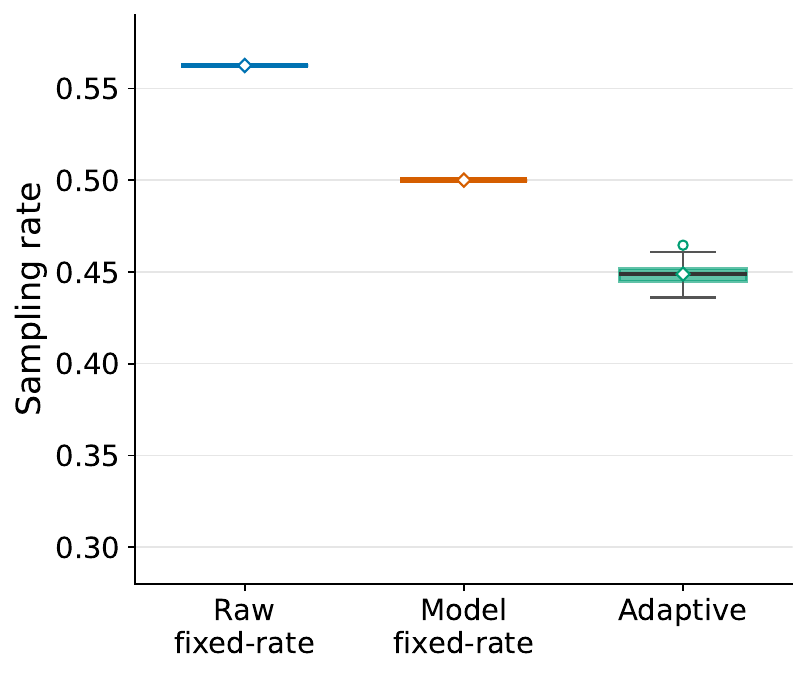}
    \caption{Sampling rate.}
\end{subfigure}
\hfill
\begin{subfigure}{0.32\textwidth}
    \centering
    \includegraphics[width=\linewidth]{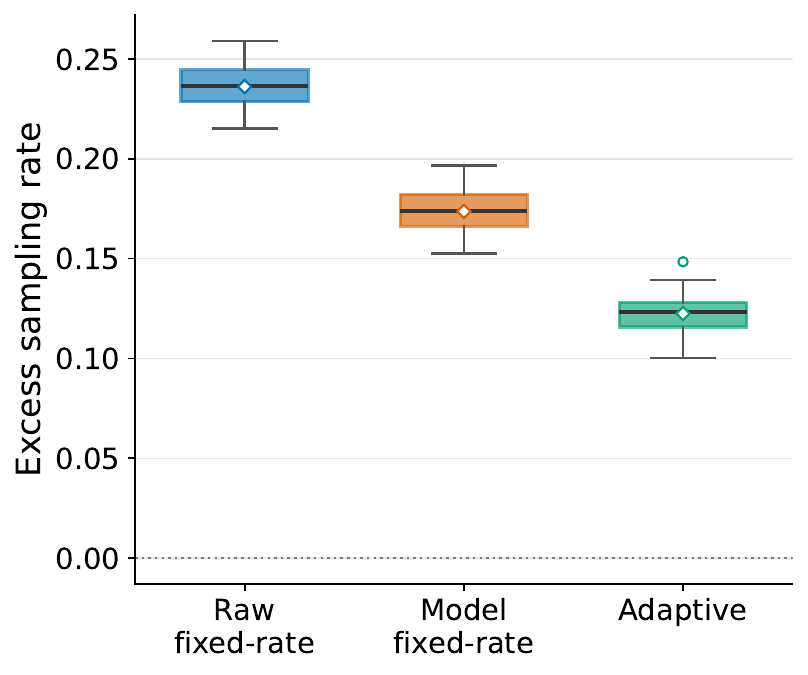}
    \caption{Excess sampling rate.}
\end{subfigure}
\caption{Performance of stopping rules on Fashion-MNIST over 50 independent runs.}
\label{fig:fashion-final-boxplots}
\end{figure}
Fashion-MNIST \citep{xiao2017fashion} consists of grayscale images of clothing items
such as sneakers, shirts, coats, and ankle boots. Each image has size \(28\times 28\);
we resize all images to \(32\times 32\) with antialiasing, so that the width matches
the 32-column Fourier acquisition grid, and all reconstructions and losses are
evaluated at this resolution. In Fashion-MNIST, even images of the same clothing item
can differ substantially in complexity, as shown in \Cref{fig:mnistimages} in
Appendix~\ref{app:mnist}. This variety makes it well suited for testing
whether the proposed rule adapts the sampling rate across test images.

As described in \Cref{sect:setup}, we sample each image starting from the
lowest-frequency column and add columns in order of frequency magnitude. The rate grid
has 32 steps, indexed by \(t=1,\ldots,32\), with \(\theta_t=t/32\), so that step \(t\)
corresponds to acquiring the \(t\) lowest-frequency columns. The reconstruction model
is trained on 2000 separate training images under the same acquisition order. The task
is to reconstruct each image with per-pixel squared error below \(c=0.003\) with
probability at least \(1-\alpha\), where \(\alpha=0.1\).
We use two separate sets of 6000 images, one for calibration and one for testing. We
set the decision time \(t_0=6\), at which 20.5\% of the test images meet the target
\(c\), so most images still require a nontrivial stopping decision.
\begin{figure}[t]
\centering
\begin{subfigure}{0.7\textwidth}
    \centering
    \includegraphics[width=\linewidth]{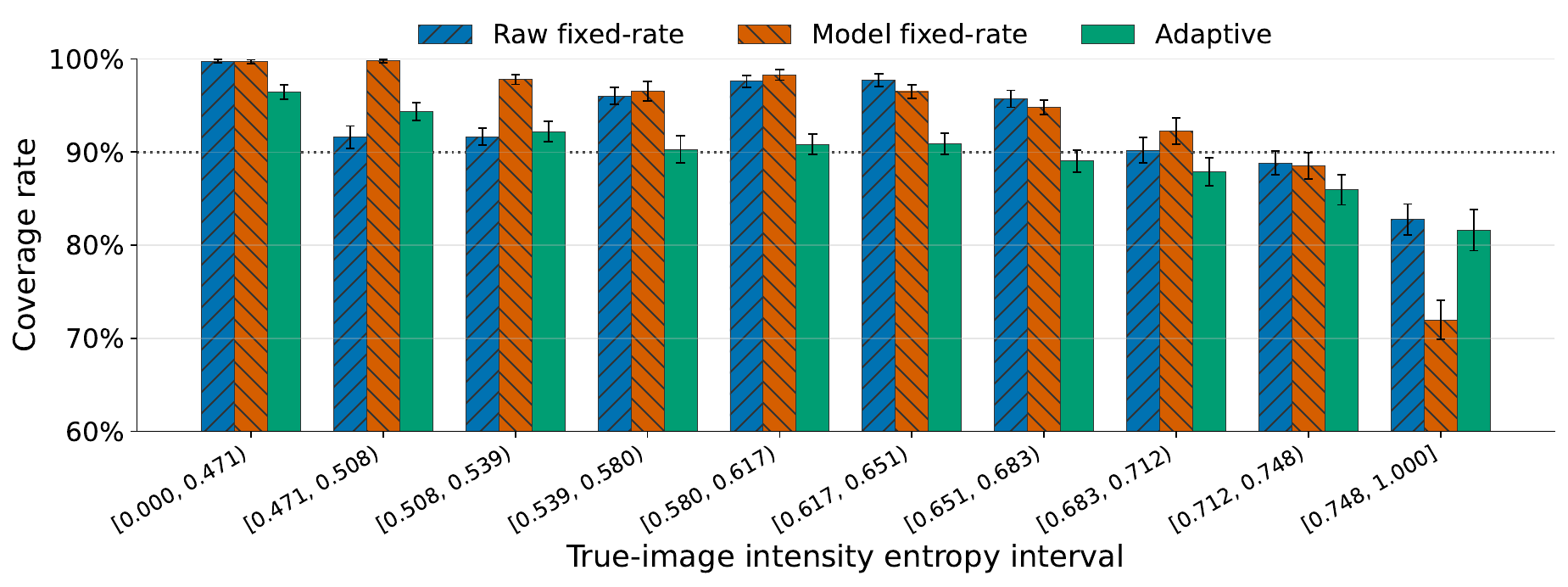}
    \caption{Coverage rate.}
\end{subfigure}
\begin{subfigure}{0.7\textwidth}
    \centering
    \includegraphics[width=\linewidth]{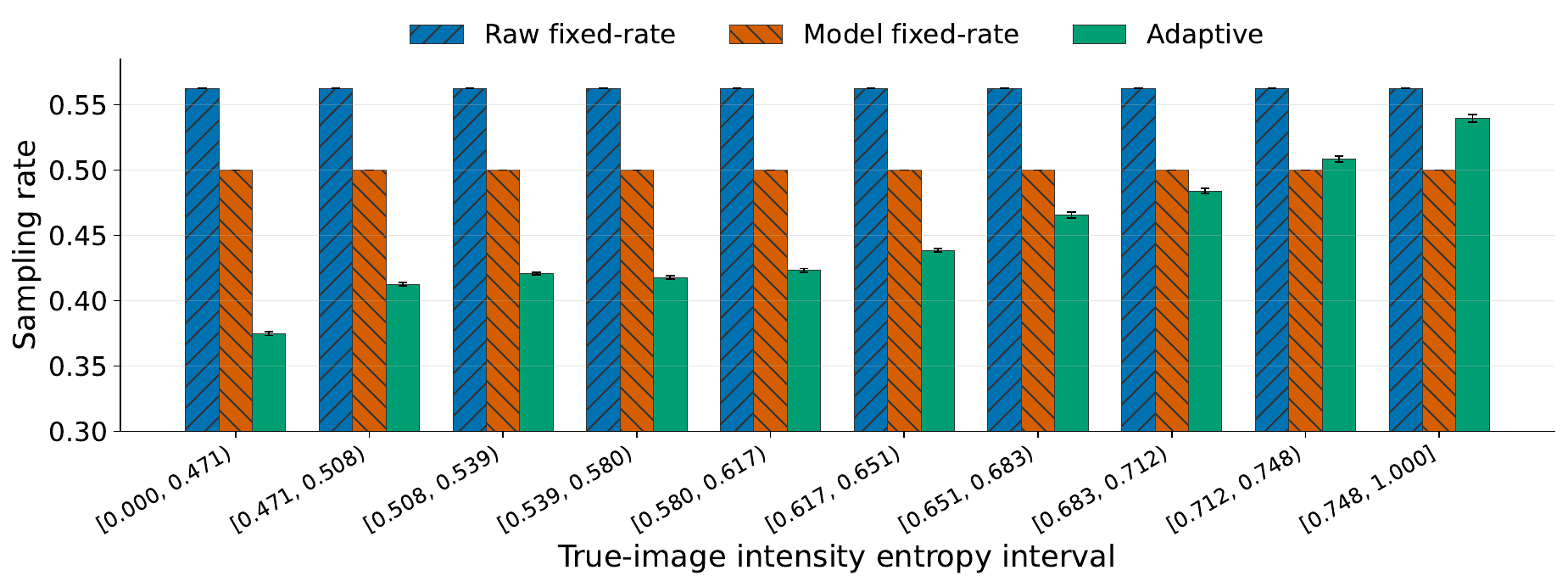}
    \caption{Sampling rate.}
\end{subfigure}
\begin{subfigure}{0.7\textwidth}
    \centering
    \includegraphics[width=\linewidth]{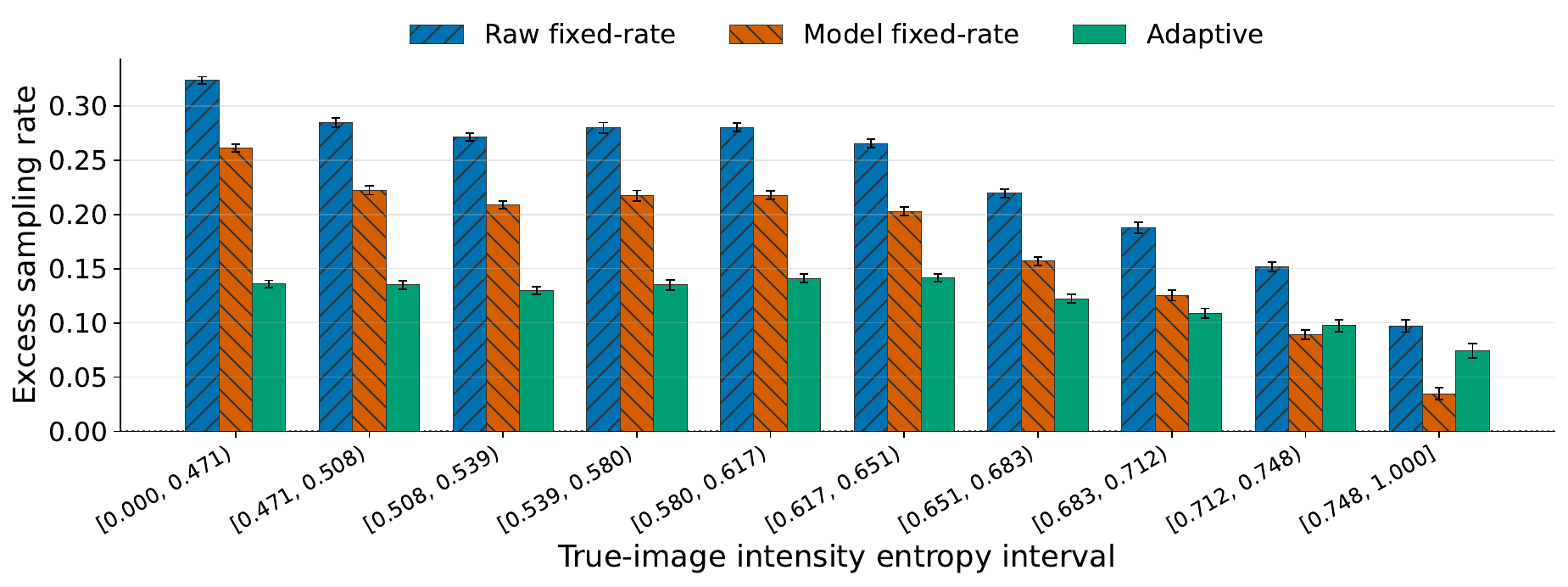}
    \caption{Excess sampling rate.}
\end{subfigure}
\caption{Performance of stopping rules over true-image entropy in Fashion-MNIST.}
\label{fig:fashion-final-entropy}
\end{figure}

\Cref{fig:fashion-final-boxplots}
reports the marginal results of the three stopping rules. Panel (a) shows that all three rules control the coverage rate at or above the target level \(1-\alpha=0.9\). The adaptive rule reaches this level with smaller sampling and excess-sampling rates than the two fixed-rate rules, as shown in panels (b) and (c).

\Cref{fig:fashion-final-entropy} evaluates the three rules across images with various complexities. The test images are divided into ten groups based on their true-image entropy \(\Omega_i\), each group containing 10\% of the test images. Panel (a) shows that the coverage of every rule stays above the target level 0.9 over most of the entropy range, but falls below 0.9 for the high-entropy images. This happens even though the adaptive rule increases its sampling rate as the true entropy grows, as shown in panel (b), while the fixed-rate rules stay constant by construction. Panel (c) shows that the adaptive rule achieves the lowest excess sampling rate over the whole entropy range. Panels (b) and (c) together show that the rule adapts in the intended direction, spending more measurements on complex images and fewer on simple ones, and that the extra sampling on complex images is not wasted, since the adaptive rule stops closest to the target stopping time at every entropy level.
\Cref{fig:fashion-final-correlation} in Appendix~\ref{app:mnist} shows that the decision-time entropy and the true-image entropy are highly correlated. The state can therefore identify complex images, and the remaining coverage gap at high entropy suggests that their large target stopping times are harder to predict.
\subsection{M4Raw brain MRI}
\begin{figure}[ht]
\centering
\begin{subfigure}{0.32\textwidth}
    \centering
    \includegraphics[width=\linewidth]{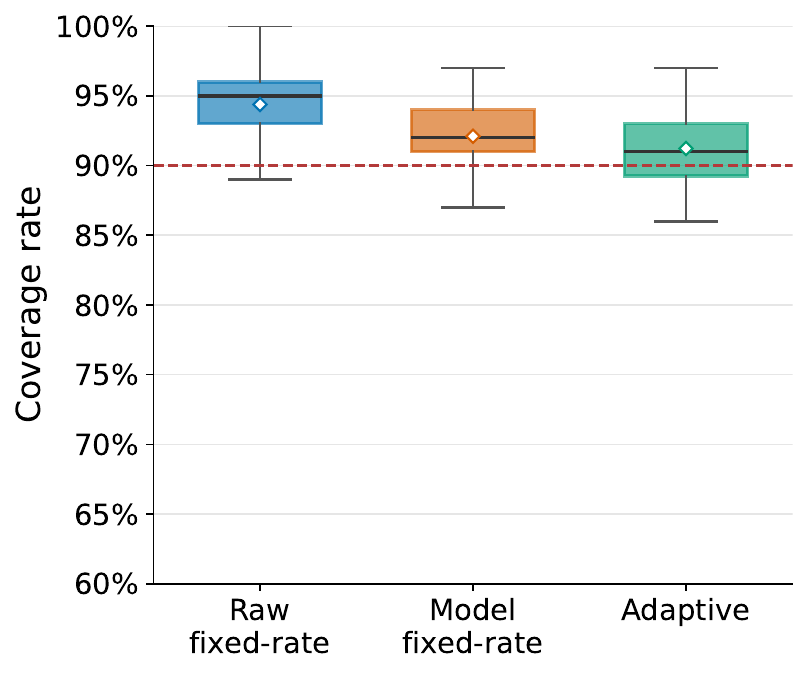}
    \caption{Coverage rate.}
\end{subfigure}
\hfill
\begin{subfigure}{0.32\textwidth}
    \centering
    \includegraphics[width=\linewidth]{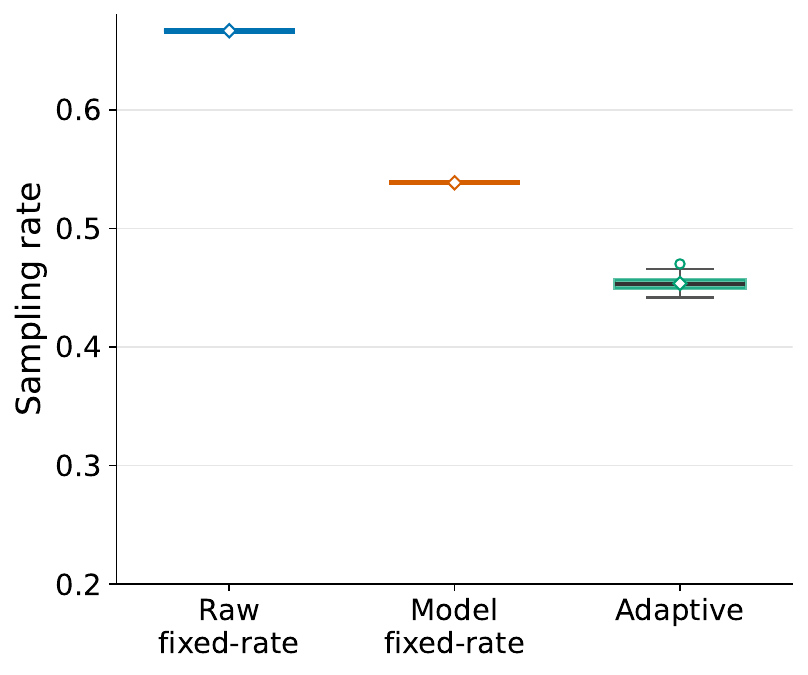}
    \caption{Sampling rate.}
\end{subfigure}
\hfill
\begin{subfigure}{0.32\textwidth}
    \centering
    \includegraphics[width=\linewidth]{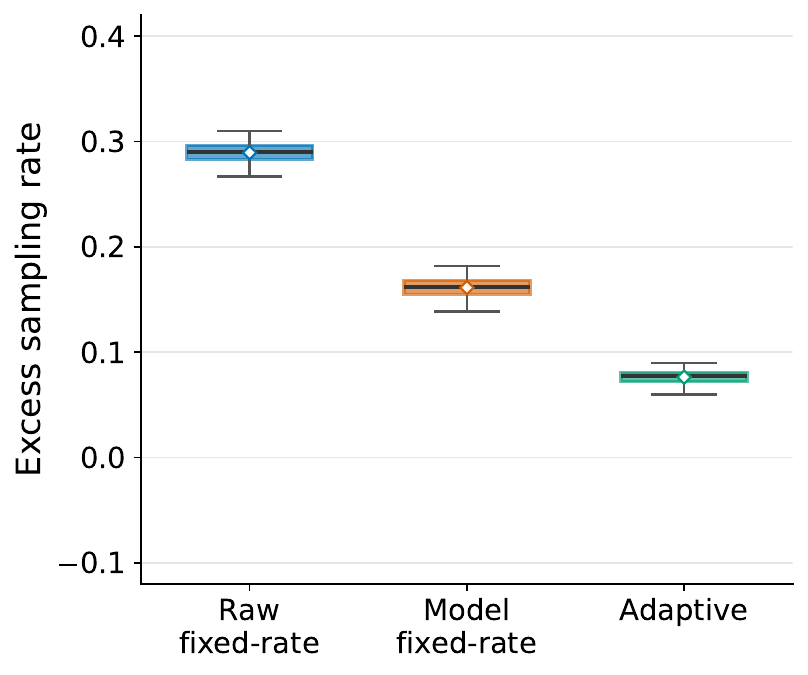}
    \caption{Excess sampling rate.}
\end{subfigure}
\caption{
Performance of stopping rules on M4Raw over 50 independent runs.}
\label{fig:m4raw-final-run-summary}
\end{figure}
M4Raw \citep{lyu2023m4raw} is a multi-coil brain MRI dataset that provides raw \(k\)-space measurements together with reference reconstructions. Each study contains 18 slices, and each slice is a \(256\times256\) image measured by four coils. The slices range from the top of the brain to the skull base, so their anatomical complexity varies widely both across and within studies; see \Cref{fig:m4raw-slices} in Appendix~\ref{sect:M4Raw} for reference slices.

We acquire each slice by adding full phase-encoding lines in order from the lowest to the highest frequency. Each slice has 195 active phase columns, and the sampling rate is the number of acquired lines divided by 195. The rate grid has 38 steps, from 16 lines up to all 195 lines.
The reconstruction model is trained on 72 studies that are disjoint from the 30 calibration studies and the 56 test studies. The task is to reconstruct each slice with per-pixel squared error below \(c=0.0015\) with probability at least \(1-\alpha\), where \(\alpha=0.1\). We use 30 studies with 540 slices for calibration and 56 disjoint studies with 1008 slices for testing, splitting at the study level so that slices from the same study never appear in both sets.
We set \(t_0=9\), corresponding to 50 acquired lines or rate \(0.256\), at which 16.7\% of the test slices meet the target \(c\). Because the multi-coil acquisition does not directly measure the band coefficients of
the RSS reference, the next-band residuals in this experiment are computed from the
fully sampled reference, and the M4Raw evaluation is retrospective
(Appendix~\ref{sect:M4Raw}).

\Cref{fig:m4raw-final-run-summary} reports the marginal results of the three stopping rules on M4Raw. Panel (a) shows that all three rules control the coverage rate at or above the target level \(1-\alpha=0.9\), and the adaptive rule reaches this level with smaller sampling and excess-sampling rates than the two fixed-rate rules, as shown in panels (b) and (c).

\begin{figure}[t]
\centering
\begin{subfigure}{0.57\textwidth}
    \centering
    \includegraphics[width=\linewidth]{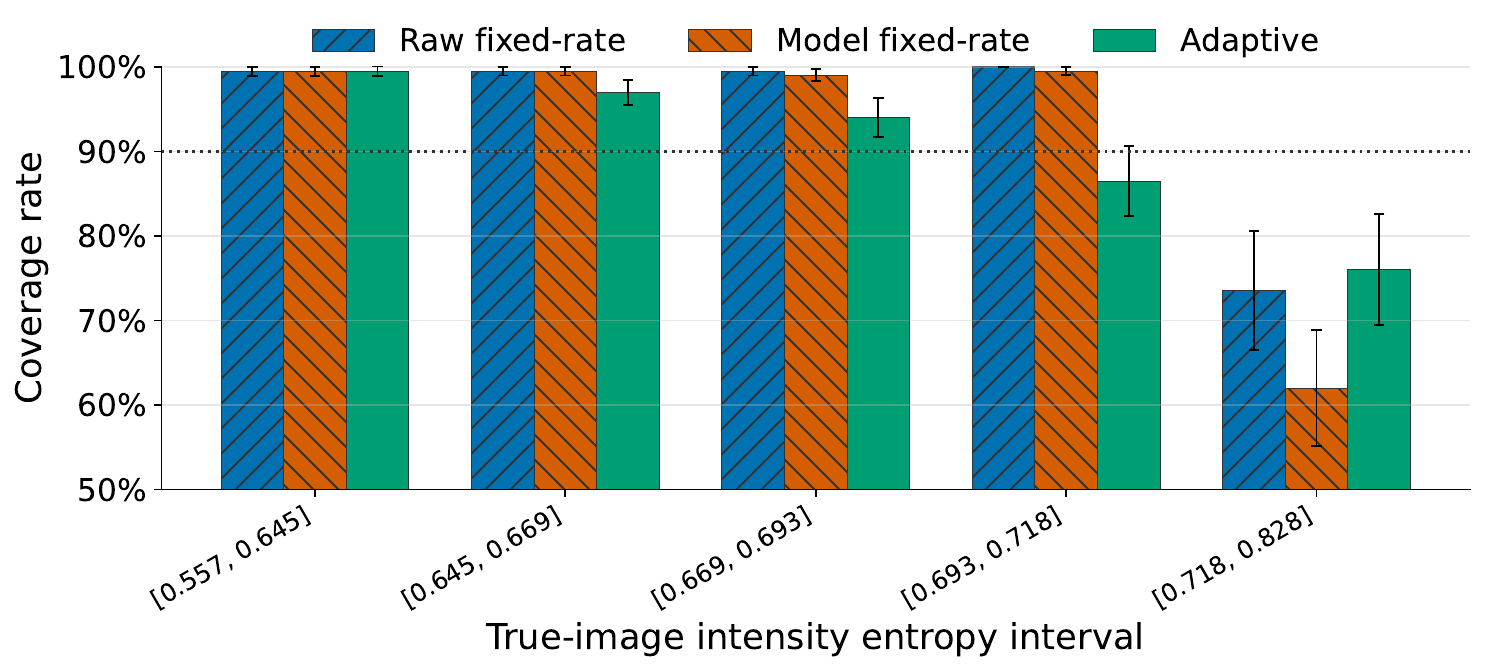}
    \caption{Coverage rate.}
\end{subfigure}
\begin{subfigure}{0.57\textwidth}
    \centering
    \includegraphics[width=\linewidth]{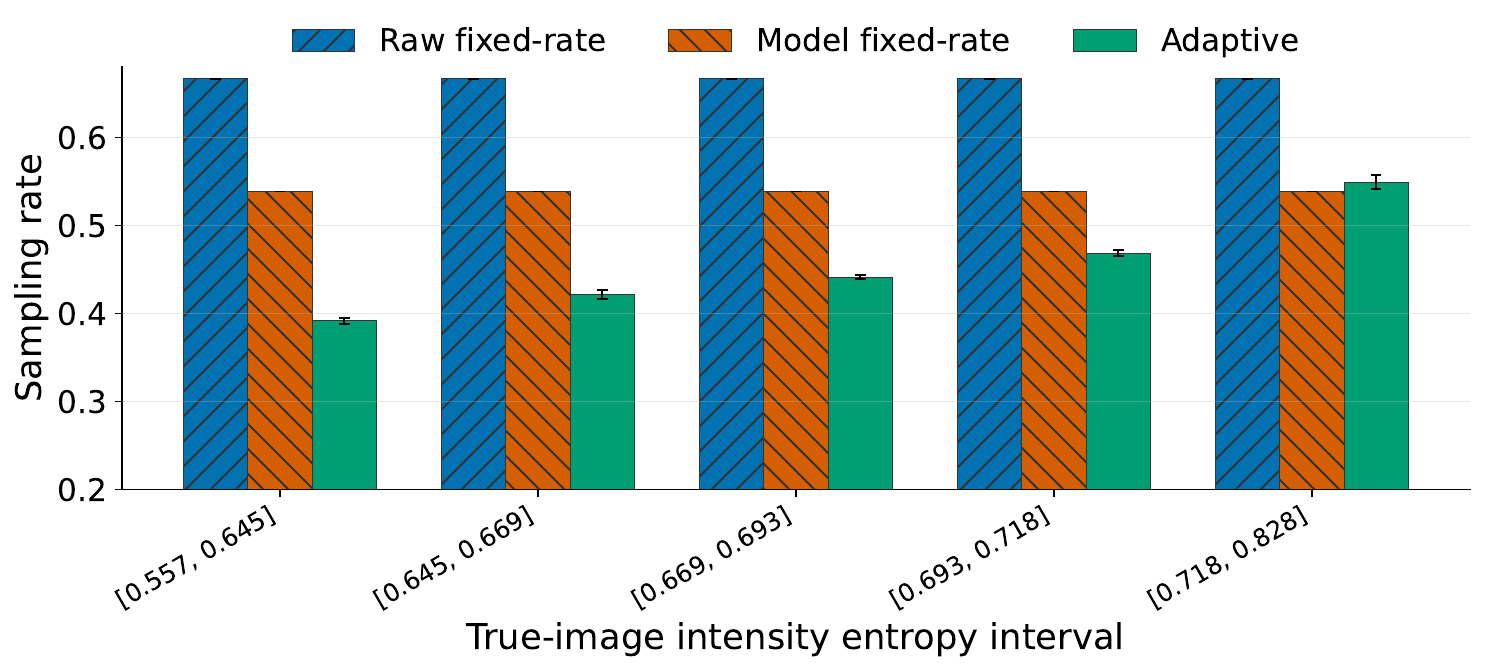}
    \caption{Sampling rate.}
\end{subfigure}
\begin{subfigure}{0.57\textwidth}
    \centering
    \includegraphics[width=\linewidth]{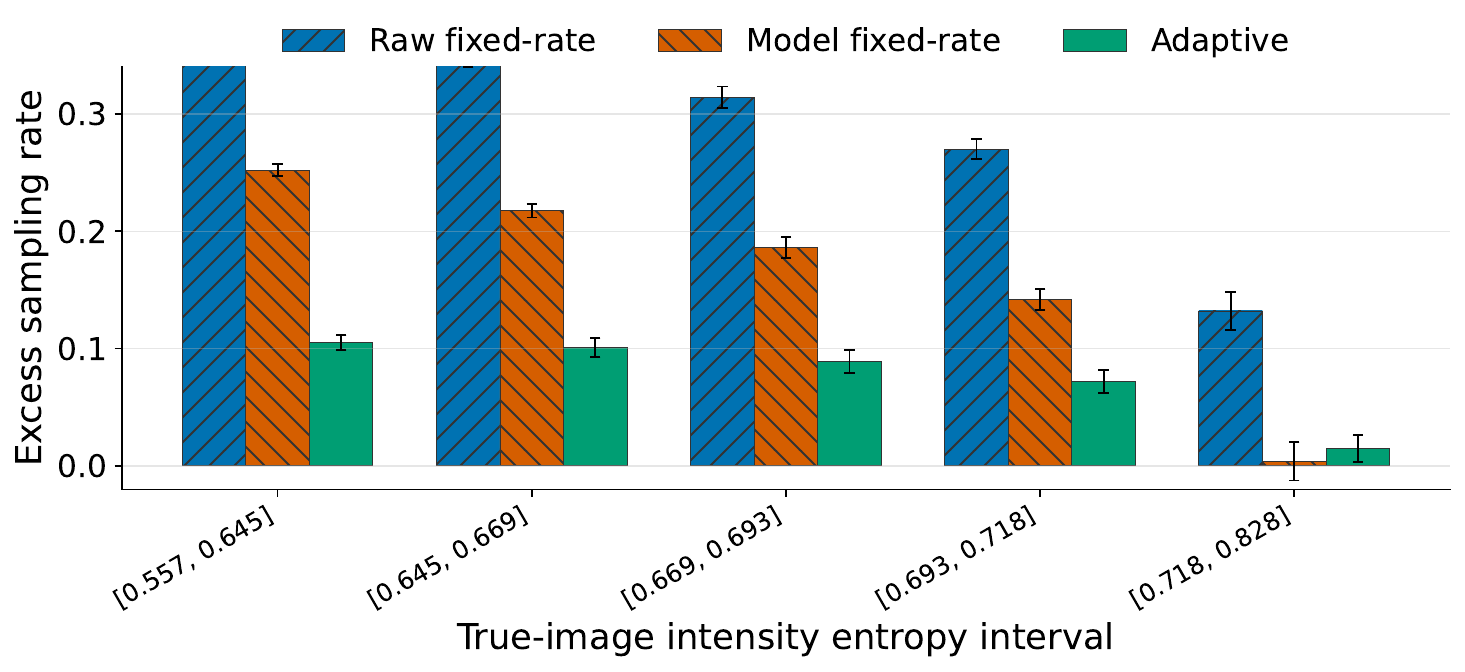}
    \caption{Excess sampling rate.}
\end{subfigure}
\caption{Performance of stopping rules over true-image entropy in M4Raw.}
\label{fig:m4raw-final-entropy}
\end{figure}

\Cref{fig:m4raw-final-entropy} evaluates the three rules across image complexities. The test slices are divided into five groups based on their true-image entropy, each group containing 20\% of the test slices. Panel (a) shows that the conditional coverage stays above the target level 0.9 over most of the entropy range and falls only for the highest-entropy slices; panel (b) shows that the adaptive rule increases its sampling rate as the true entropy grows,  and panel (c) shows that the adaptive rule keeps the lowest excess sampling rate over nearly the whole entropy range. As on Fashion-MNIST, panels (b) and (c) together show that the extra sampling on complex slices is well spent.
\Cref{fig:m4raw-final-correlation} in Appendix \ref{sect:M4Raw} shows that the decision-time entropy and the true-image entropy are also highly correlated in M4Raw. The coverage gap at high entropy reflects the difficulty of predicting the target stopping times of complex slices.

\section{Discussion}
\label{sec:discussion}
This paper introduced Conformalized Rate-Adaptive Sensing (CoRAS), a method that chooses the acquisition rate one image at a time while keeping the reconstruction error below a target level with high probability. CoRAS extrapolates the observed reconstruction path to predict the target stopping time, then corrects that prediction using calibration images with similar states. In the experiments on Fashion-MNIST and M4Raw brain MRI, CoRAS achieved the target stopping-time coverage while spending fewer measurements on average than the fixed-rate rules. This gain comes from taking more measurements on the complex images and fewer on the simple ones.

One limitation of CoRAS is that the decision time \(t_0\) is fixed and the rule stops only once. A more flexible design would revisit the decision at several points along the acquisition path, updating the state each time and asking again whether to continue. Reusing conformal prediction this way may break the finite-sample guarantee, so developing sequentially valid stopping rules is a natural next step.

The construction of CoRAS can reduce cost wherever measurement or computation is bought one step at a time. A large language model can reason for longer before it answers, and one would like to stop as soon as the answer is good enough. In spectroscopy, a sample is scanned repeatedly and the scans are averaged, and one would like to stop once the spectrum is clean enough to identify the compound. Both problems have the structure CoRAS calibrates: a cost paid along a path, and a decision made before the outcome is seen. What CoRAS would add in either case is a stopping rule that carries a guarantee, rather than a threshold set by hand.

\bibliographystyle{plainnat}
\bibliography{references}

\clearpage

\appendix
\appendixpage

\section{Experiment details}
\label{app:experiment_details}

\subsection{Fashion-MNIST}\label{app:mnist}

The reconstruction model is a convolutional network trained for ten epochs on the 2000 training images, with sampling rates drawn from the acquisition grid. \Cref{fig:mnistimages} shows reference test images from the clothing classes used in this experiment.

\begin{figure}[H]
\centering
\includegraphics[width=0.9\textwidth]{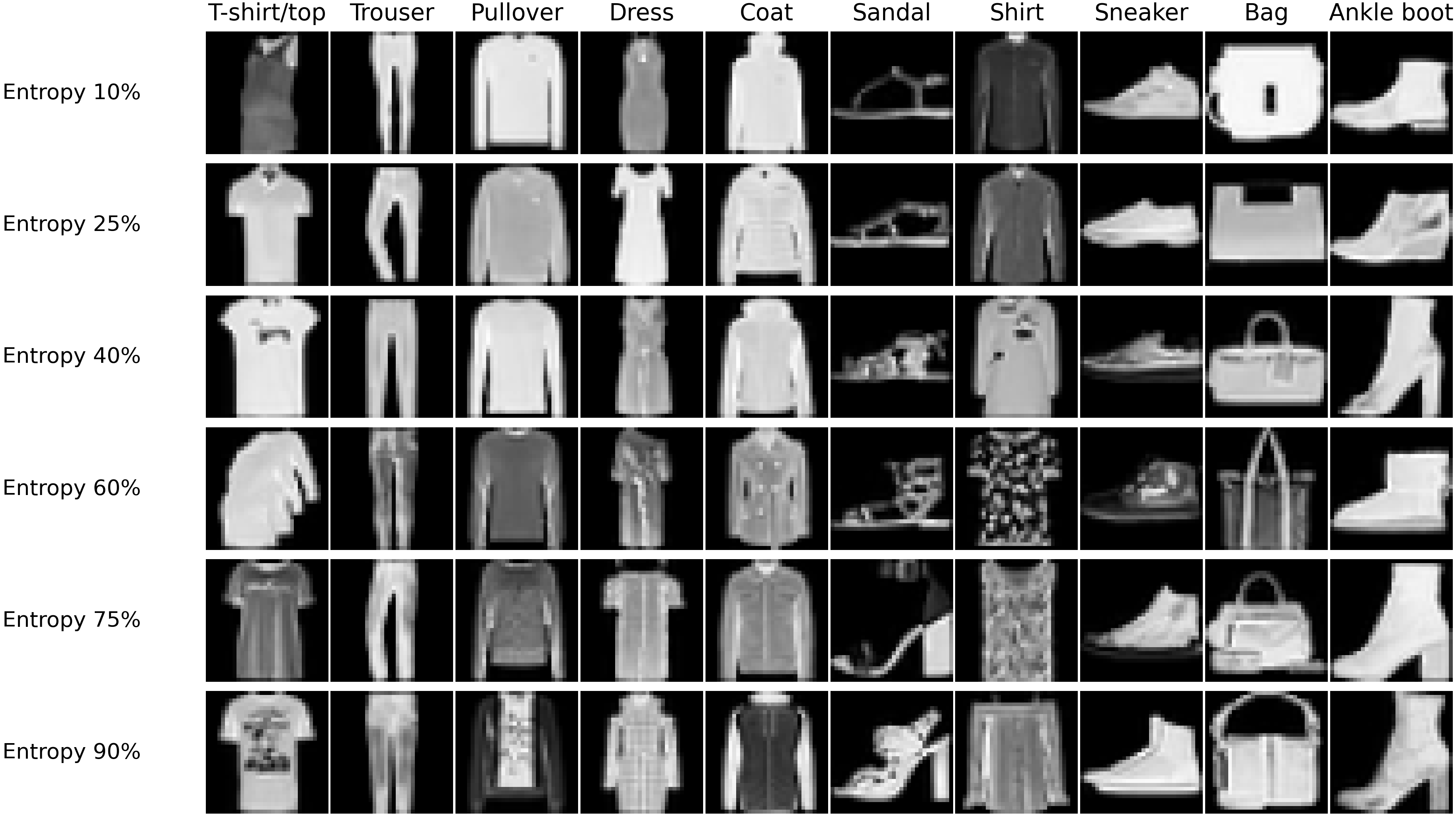}
\caption{Reference Fashion-MNIST test images spanning the clothing classes.}
\label{fig:mnistimages}
\end{figure}

\paragraph{State and bandwidth selection.}
The state uses the horizontal prediction \(\widehat T_i^H\), which lies in
\([t_0,t_{\max}]\) by \eqref{equ:T_H}, and the 16-bin entropy \(\widehat\Omega_i\) of
the reconstruction at \(t_0\). The horizontal coordinate is mapped to \([0,1]\) as
\((\widehat T_i^{H}-t_0)/(t_{\max}-t_0)\), and the entropy coordinate is already
normalized to \([0,1]\). On these prespecified coordinates, the distance in \eqref{eq:anisotropic_vertical_distance} uses the squared bandwidths
\[
(h_T^2,h_\Omega^2)\in \{0.0125\times2^m:m=-4,\ldots,4\}^2.
\]
For each test image and candidate value of the stopping time, the rule reselects the bandwidth pair using the online algorithm described in Appendix \ref{app:efficient_online_bandwidth}.

\Cref{fig:fashion-final-correlation} plots the entropy of the reconstruction at \(t_0\) against the true-image entropy on the 6000 test images. The two are strongly correlated, with Pearson correlation \(0.935\) and Spearman correlation \(0.931\). The decision-time entropy is therefore a reliable proxy for the complexity of the true image, which is not observed at test time, and this supports its use in the state \(S_i\).

\begin{figure}[t]
\centering
\includegraphics[
 width=0.9\textwidth,
    trim={0 0 0 7cm},
    clip
]{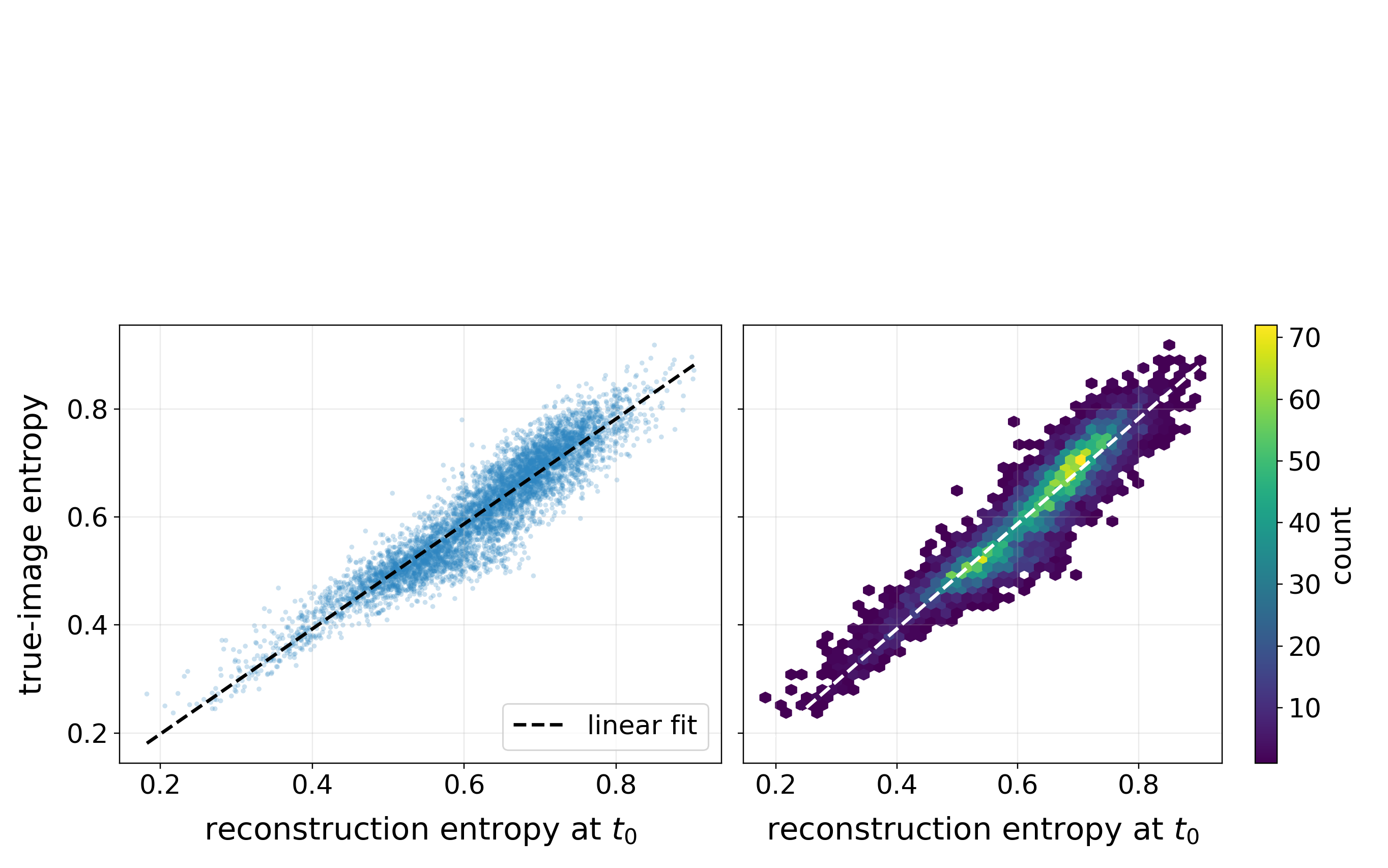}
\caption{Reconstruction entropy at \(t_0\) versus true-image entropy on the 6000 Fashion-MNIST test images. Left: scatter plot with a linear fit. Right: hexagonal bin density.}
\label{fig:fashion-final-correlation}
\end{figure}

\subsection{M4Raw brain MRI}\label{sect:M4Raw}

The reference reconstructions are the root-sum-of-squares (RSS) images provided with the dataset. For every slice, the RSS target and the zero-filled reconstructions are scaled by their 99.5th intensity percentile and clipped to \([0,1]\). \Cref{fig:m4raw-slices} shows reference slices from eight test studies.

\begin{figure}[ht]
\centering
\includegraphics[width=0.9\textwidth]{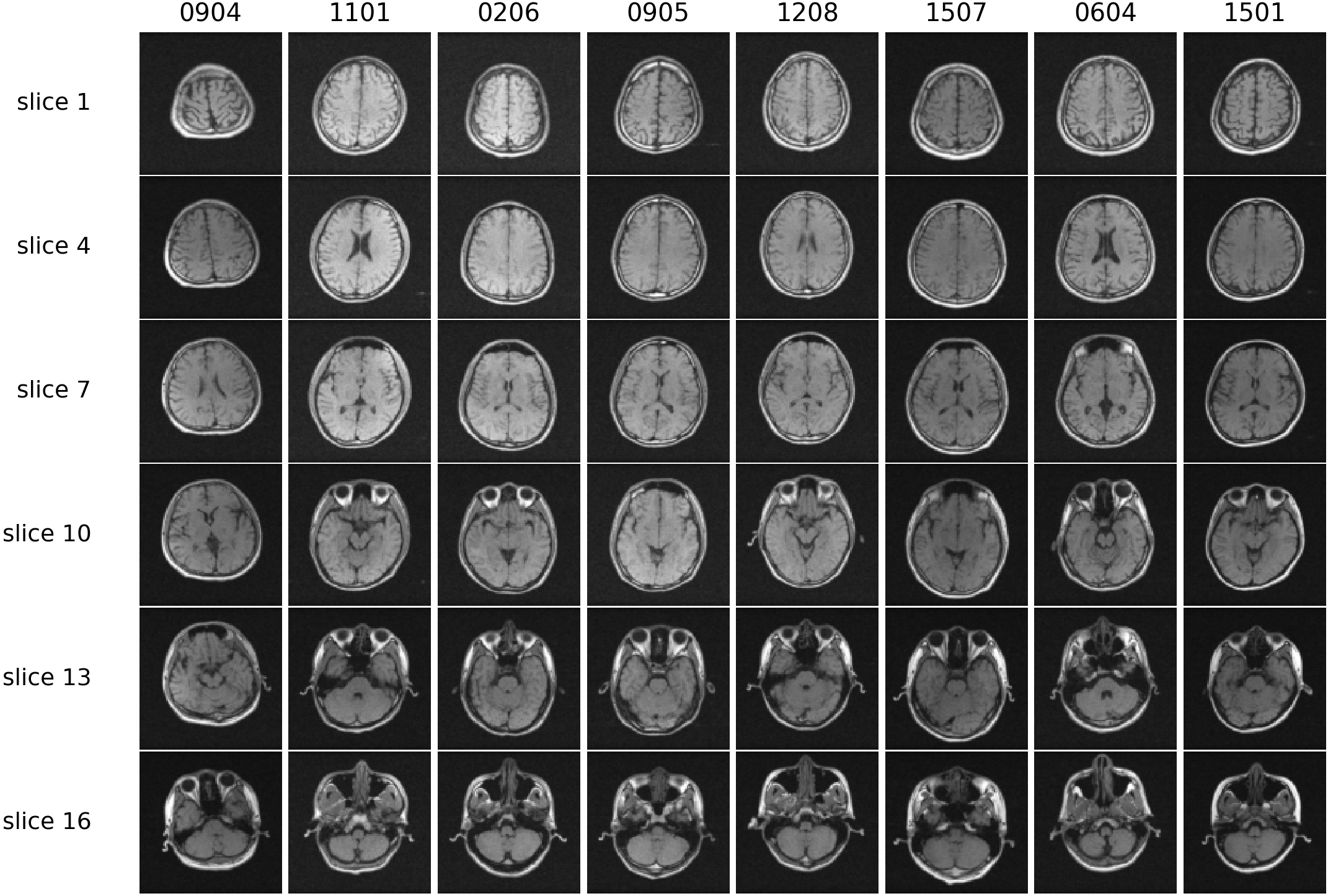}
\caption{Reference RSS reconstructions from eight M4Raw test studies (columns) at six slice positions (rows). Slices vary from the top of the brain to the skull base, and anatomical complexity varies both across slice positions and across studies.}
\label{fig:m4raw-slices}
\end{figure}

The acquisition path contains 38 nested line counts,
\[
\begin{split}
&16,20,24,28,32,36,40,45,50,55,60,65,70,75,80,85,90,95,100,105,110,115,120,\\
&125,130,135,140,145,150,155,160,165,170,175,180,185,190,195.
\end{split}
\]
\paragraph{Next-band residuals.}
Let \(Y_i^{\mathrm{RSS}}\in[0,1]^{H\times W}\) denote the normalized fully sampled RSS
reference for slice \(i\), where \(H=W=256\) are the image height and width, and let
\(\widehat Y_i^{\mathrm{RSS}}(t-1)\) be the single-channel model reconstruction at
step \(t-1\). Let \(\mathcal N_t\) denote the set of phase columns newly acquired at
step \(t\), and \(P_{\mathcal N_t}\) the restriction of a Fourier array to those
columns. With \(\mathcal F_{\mathrm c}\) the centered orthonormal two-dimensional
Fourier transform, we define
\[
I_{i,t}^{\mathrm{MRI}}
=
\frac{1}{K}
\left\|
P_{\mathcal N_t}
\mathcal F_{\mathrm c}
\left\{
Y_i^{\mathrm{RSS}}-\widehat Y_i^{\mathrm{RSS}}(t-1)
\right\}
\right\|_2^2,
\qquad
K=HW.
\]
This is the single-channel definition \eqref{eq:next_band_prediction_residual} applied
to the RSS image, with \(\Delta P_t\) the projection onto the newly added phase
columns; it replaces \(I_{i,t}\) in \eqref{equ:def_H}. The normalization by the full
dimension \(K\) follows the same convention as
\eqref{eq:next_band_prediction_residual}, so the band energies and the extrapolated
tail \(\hat L_i^H(t)\) are on the same scale as the reconstruction loss and the
threshold \(c\). The acquisition increments contain four or five columns, and a wider
increment is intentionally allowed to contribute more residual energy.
Unlike the single-channel Fourier setting, the Fourier coefficients of the RSS image in
the new columns are not directly measured by the multi-coil acquisition, because the
RSS combination is nonlinear across coils. The M4Raw experiment therefore computes
\(I_{i,t}^{\mathrm{MRI}}\) from the fully sampled reference, and the evaluation is
retrospective.

The reconstruction model is a residual-gated U-Net. The model is trained for 15 epochs with a squared loss, together with monotonicity and zero-filled anchor penalties. Its three input channels are the normalized zero-filled RSS reconstruction \(X_{\mathrm{zf}}(\theta)\), the sampling mask \(M_\theta\), and a constant \(\theta\)-plane. The model output is
\[
\hat Y_\phi(\theta)
=
\operatorname{clip}_{[0,1]}
\left[
X_{\mathrm{zf}}(\theta)
+(1-\theta)\,
\Delta_\phi\{X_{\mathrm{zf}}(\theta),M_\theta,\theta\}
\right],
\]
where \(\Delta_\phi\) is the learned residual. The network therefore does not predict the image from scratch; it predicts a correction to the zero-filled reconstruction, which already carries the acquired measurements. The gate \((1-\theta)\) shrinks that correction as the sampling rate grows, so the output returns to the zero-filled reconstruction when the acquisition is nearly complete and the correction is no longer needed. The clip keeps the output in the intensity range \([0,1]\) of the normalized images. 

\paragraph{State and bandwidth selection.}
The adaptive rule uses the same two-dimensional state and the same exact candidate-specific bandwidth selection as in the Fashion-MNIST experiment, applied to the 540 calibration slices. 
The state uses the horizontal prediction \(\widehat T_i^H\) and the 16-bin entropy
\(\widehat\Omega_i\) of the reconstruction at \(t_0\). The horizontal coordinate is
mapped to \([0,1]\) as \((\widehat T_i^{H}-t_0)/(t_{\max}-t_0)\), and the entropy
coordinate is already normalized to \([0,1]\).
Both squared bandwidths use the prespecified grid
\(\mathcal B=\{0.0125\times 2^m:m=-4,\ldots,4\}\).
For each test slice and candidate value \(t\), the rule reselects
\((h_T^2,h_\Omega^2)\in\mathcal B^2\) using the algorithm in
Appendix \ref{app:efficient_online_bandwidth}.

As in the Fashion-MNIST experiment, the entropy of the reconstructed images at the decision time \(t_0\) is strongly correlated with the true-image entropy across the M4Raw test slices, as shown in \Cref{fig:m4raw-final-correlation}. The decision-time entropy is an informative proxy for slice complexity, which supports its use as a state coordinate for the MRI data.

\begin{figure}[H]
\centering
\includegraphics[
    width=0.9\textwidth,
    trim={0 0 0 13cm},
    clip
]{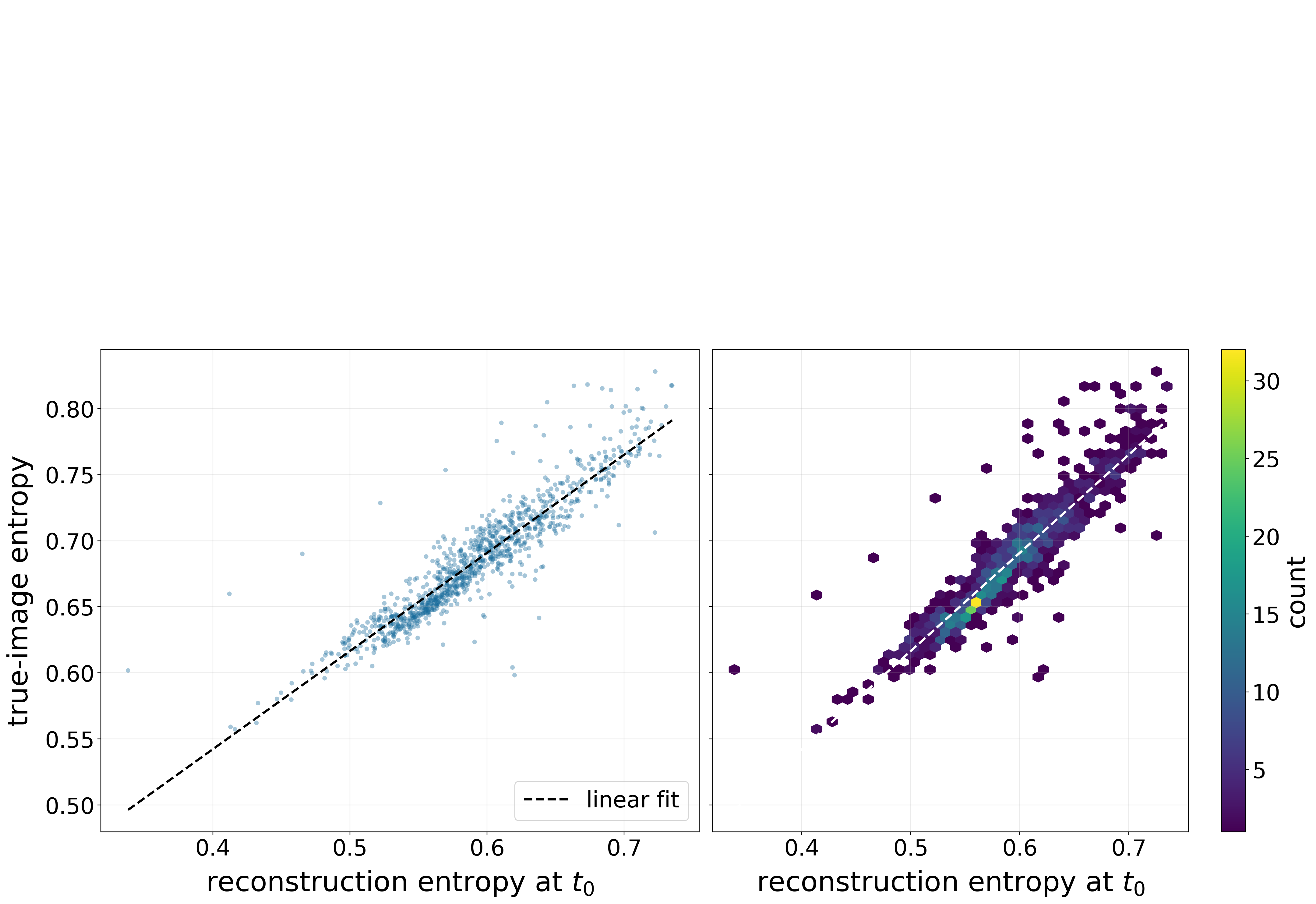}
\caption{Reconstruction entropy at \(t_0\) versus true-image entropy on the 1008 M4Raw test slices. Left: scatter plot with a linear fit. Right: hexagonal bin density.}
\label{fig:m4raw-final-correlation}
\end{figure}

\section{Additional results for the horizontal prediction}
\label{app:horizontal_extra}

This appendix collects two results that support the horizontal prediction model presented in
\Cref{subsec:horizontal}. Appendix \ref{app:ordered_tail} extends
\Cref{prop:next_band_residual_power_law_decay} to the case where the tail exponent varies
with acquisition time, and Appendix \ref{app:monotone_projection} shows how the monotonicity
constraint sharpens the extrapolation error bound.

\subsection{Local polynomial trend in the ordered-tail model}
\label{app:ordered_tail}

\Cref{prop:next_band_residual_power_law_decay} treats the stationary case, in which
\(p_i(\tau)\) does not depend on the acquisition time. The following proposition covers a
time-varying tail exponent.

\begin{proposition}
\label{prop:preprojection_residual_polynomial_trend}
Suppose \Cref{ass:ordered_tail_preprojection_residual} holds and the cutoffs are
logarithmically spaced as in \eqref{equ:omegat}. Fix an order \(q\ge1\). Assume that
\(p_i(\tau)\) is \(q+1\) times continuously differentiable in a neighborhood of
\(\tau=t_0-1\), with \(p_i(\tau)>1/2\) in this neighborhood. Then there exist
coefficients \(\gamma_{i,0}^{(q)},\ldots,\gamma_{i,q}^{(q)}\) such that, for \(t\) near
\(t_0\),
\[
\log I^\star_{i,t}
=
\sum_{\ell=0}^q
\gamma_{i,\ell}^{(q)}(t-t_0)^\ell
+
O_q\!\left(|t-t_0|^{q+1}\right).
\]
Moreover, for neighboring acquisition steps \(t,t+1\) near \(t_0\), if
\(p_i(t)\ge p_i(t-1)\), then
\[
I^\star_{i,t+1}\le I^\star_{i,t}.
\]
\end{proposition}

\Cref{prop:preprojection_residual_polynomial_trend} shows that when \(p_i(\tau)\) varies
smoothly with acquisition time, the next-band residual \(\log I^\star_{i,t}\) admits a
local polynomial approximation near the decision time.

The second claim gives a condition for a decreasing trend of
\(I^\star_{i,t}\). If \(p_i(t)\ge p_i(t-1)\), then the weight \(x^{-p_i(\tau)}\) in
\eqref{equ:E_i} decays no more slowly after one additional acquisition step, i.e.,
the residual tail becomes no flatter from step \(t-1\) to step \(t\).

This condition need not hold at every acquisition step. Reconstruction artifacts,
texture, and image-specific fine-scale structure can create local violations. The point
is that, along an ordered coarse-to-fine acquisition path, this decay condition may hold
for most steps, so monotonicity captures the dominant trend of the residual history. The
monotone polynomial fit is therefore used to estimate this dominant decaying trend rather
than to enforce a property that holds exactly.

\subsection{Prediction gain from monotone projection}
\label{app:monotone_projection}

We next compare the unconstrained polynomial fit with the monotone fit at the same
image-wise selected order \(\hat q_i\) in \Cref{subsec:horizontal}. The result
shows that projecting onto the monotone class sharpens the extrapolation
bound on the future index set \(\mathcal R_0^{\mathrm c}\).

For a coefficient vector \(\gamma\in\Gamma_r\), define
\[
\|\gamma\|_{\mathcal R_0,r}^2
:=
\sum_{t\in\mathcal R_0}\{\phi_r(t)^\top\gamma\}^2,
\qquad
\|\gamma\|_{\mathcal R_0^{\mathrm c},r}^2
:=
\sum_{t\in\mathcal R_0^{\mathrm c}}\{\phi_r(t)^\top\gamma\}^2 .
\]
For the selected order \(\hat q_i\), define the extrapolation constant
\[
C_{\mathrm{ext}}(\hat q_i)
:=
\sup_{\gamma\in\Gamma_{\hat q_i}\setminus\{0\}}
\frac{
\|\gamma\|_{\mathcal R_0^{\mathrm c},\hat q_i}^2
}{
\|\gamma\|_{\mathcal R_0,\hat q_i}^2
}.
\]
This ratio is finite because \(\Gamma_{\hat q_i}\) is finite-dimensional and the design on
\(\mathcal R_0\) has full column rank.

\begin{proposition}
\label{prop:monotone_projection_gain}
Let \(\hat q_i\in\mathcal Q\) be the selected order for image \(i\). Suppose the true
log-residual trend has order \(q_i^\star\le \hat q_i\), meaning that
\[
\mu_i^\star(t)=\phi_{\hat q_i}(t)^\top\gamma_i^\star
\]
for some \(\gamma_i^\star\in\Gamma_{\hat q_i}^{\mathrm{mon}}\), with lower-order
coefficients padded by zeros if needed. Let \(\hat\gamma_{i,\hat q_i}\) be the
unconstrained least-squares fit over \(\Gamma_{\hat q_i}\). Then
\[
\|\hat\gamma_{i,\hat q_i}-\gamma_i^\star\|_{\mathcal R_0^{\mathrm c},\hat q_i}^2
\le
C_{\mathrm{ext}}(\hat q_i)
\|\hat\gamma_{i,\hat q_i}-\gamma_i^\star\|_{\mathcal R_0,\hat q_i}^2 .
\]
Let \(\hat\gamma_{i,\hat q_i}^{\mathrm{mon}}\) be the constrained least-squares
fit over \(\Gamma_{\hat q_i}^{\mathrm{mon}}\). Then
\[
\|\hat\gamma_{i,\hat q_i}^{\mathrm{mon}}-\gamma_i^\star\|_{\mathcal R_0^{\mathrm c},\hat q_i}^2
\le
C_{\mathrm{ext}}(\hat q_i)
\left\{
\|\hat\gamma_{i,\hat q_i}-\gamma_i^\star\|_{\mathcal R_0,\hat q_i}^2
-
\|\hat\gamma_{i,\hat q_i}
-
\hat\gamma_{i,\hat q_i}^{\mathrm{mon}}\|_{\mathcal R_0,\hat q_i}^2
\right\}.
\]
\end{proposition}

The second bound follows because \(\Gamma_{\hat q_i}^{\mathrm{mon}}\) is convex, so
\(\hat\gamma_{i,\hat q_i}^{\mathrm{mon}}\) is the \(\|\cdot\|_{\mathcal R_0,\hat q_i}\)
projection of the unconstrained fit onto the monotone class, and \(\gamma_i^\star\) lies in
that class. The two displays compare deterministic upper bounds rather than realized
prediction errors; the proposition does not by itself imply that the constrained
fit has smaller realized extrapolation error in every dataset.

\Cref{prop:monotone_projection_gain} assumes \(q_i^\star\le\hat q_i\). This means
that the selected image-wise polynomial class is rich enough to contain the true dominant
trend. If \(\hat q_i<q_i^\star\), the fit has approximation bias; an additional bias
term would be needed, and the projection argument alone does not guarantee accurate
extrapolation.

When \(\hat q_i\ge q_i^\star\), the true trend can be represented by the selected
class, but the unconstrained fit can still be unstable. A larger selected order can
absorb local fluctuations in the short pre-decision history. For example, fix an order \(r\) and suppose
\[
H_{i,t}
=
\phi_r(t)^\top\gamma_i^\star+\varepsilon_{i,t},
\qquad
t\in\mathcal R_0,
\]
where the errors are independent and homoskedastic with variance
\(\sigma^2\). Then
\[
\mathbb E\left[
\|\hat\gamma_{i,r}-\gamma_i^\star\|_{\mathcal R_0,r}^2
\right]
=
\sigma^2(r+1).
\]
Thus, before accounting for data-dependent order selection, increasing the
candidate order increases the amount of noise that an unconstrained
least-squares fit can absorb. A larger order can reduce approximation bias,
but it can also increase variance.

The monotone fit addresses this overfitting problem. The term
\(
\|\hat\gamma_{i,\hat q_i}
-
\hat\gamma_{i,\hat q_i}^{\mathrm{mon}}\|_{\mathcal R_0,\hat q_i}^2
\)
measures the part of the unconstrained fit that is removed by projecting onto the
monotone class. More specifically, this term captures local upward movements that are
inconsistent with the dominant decaying trend. The proposition shows that removing this
component sharpens the extrapolation bound.

In summary, the image-wise order \(\hat q_i\) controls the bias--variance trade-off, while
the monotonicity constraint stabilizes the selected-order fit. If \(\hat q_i\) is too
small, the method may underfit the true trend. If \(\hat q_i\) is large enough, the monotone
projection can improve extrapolation by removing nonmonotone fluctuations.

\section{Bandwidth selection}
\label{app:bandwidth_selection}

Recall that 
\(
S_i
=
\left(
\hat T_i^H,\hat\Omega_i
\right)^\top,
\)
where \(\hat T_i^H\) is the horizontal stopping-time prediction and
\(\hat\Omega_i\) is the decision-time reconstruction entropy.
For bandwidths \(h_T>0\) and \(h_\Omega>0\), 
\begin{equation}
\label{eq:vertical_bandwidth_distance}
\ell_{jk}(h_T,h_\Omega)
=
\frac{
(\hat T_j^H-\hat T_k^H)^2
}{
h_T^2
}
+
\frac{
(\hat\Omega_j-\hat\Omega_k)^2
}{
h_\Omega^2
}.
\end{equation}
A smaller \(h_T\) requires tighter matching in the horizontal prediction, while a
smaller \(h_\Omega\) requires tighter matching in reconstruction entropy. Conversely,
larger bandwidths produce more diffuse weights and a more global vertical correction.

For an index set \(\mathcal I\subseteq[n+1]\), define
\begin{equation}
\label{eq:vertical_weights_index_set}
\hat\omega_{j,k}
(h_T,h_\Omega;\mathcal I)
=
\frac{
\exp\{-\ell_{jk}(h_T,h_\Omega)\}
}{
\displaystyle
\sum_{l\in\mathcal I\setminus\{j\}}
\exp\{-\ell_{jl}(h_T,h_\Omega)\}
},
\qquad
k\in\mathcal I\setminus\{j\}.
\end{equation}

Let \(\mathcal H_T\) and \(\mathcal H_\Omega\) be fixed finite grids of candidate
bandwidths. In experiments, the squared bandwidths are chosen from
\[
\mathcal B
=
\left\{
0.0125\times2^m:m=-4,\ldots,4
\right\},
\qquad
\mathcal H_T
=
\mathcal H_\Omega
=
\left\{
\sqrt b:b\in\mathcal B
\right\}.
\]

\subsection{Exact candidate-specific bandwidth selection}\label{sect:exactband}

For a candidate stopping time
\(t\in\{t_0,\ldots,t_{\max}\}\), recall the augmented floored labels
\[
T_k^\star(t)
=
\begin{cases}
T_k^\star, & k=1,\ldots,n,\\
t, & k=n+1.
\end{cases}
\]
For a fixed bandwidth pair \((h_T,h_\Omega)\), define the augmented horizontal
residuals
\[
e_k(t)
:=
T_k^\star(t)-\hat T_k^H,
\qquad
k=1,\ldots,n+1.
\]
The horizontal--vertical prediction for augmented point \(j\) is
\begin{equation}
\label{eq:augmented_hv_prediction_bandwidth}
\hat T_j^{HV}(t;h_T,h_\Omega)
=
\hat T_j^H
+
\sum_{k\in[n+1]\setminus\{j\}}
\hat\omega_{j,k}
(h_T,h_\Omega;[n+1])
e_k(t).
\end{equation}
Its corresponding residual is
\[
R_j(t;h_T,h_\Omega)
=
T_j^\star(t)-\hat T_j^{HV}(t;h_T,h_\Omega).
\]
We select the bandwidth pair for every candidate
\(t\):
\begin{equation}
\label{eq:augmented_bandwidth_selection}
\bigl(
\hat h_T(t),\hat h_\Omega(t)
\bigr)
\in
\argmin_{
h_T\in\mathcal H_T,
h_\Omega\in\mathcal H_\Omega
}
Q_{h_T,h_\Omega}(t),
\end{equation}
where $Q_{h_T,h_\Omega}(t)$ is the squared loss for predicting the stopping time:
\begin{equation}
    \label{eq:augmented_bandwidth_objective}
Q_{h_T,h_\Omega}(t)
:=
\frac{1}{n+1}
\sum_{j=1}^{n+1}
R_j(t;h_T,h_\Omega)^2.
\end{equation}
The resulting vertical weights are defined as in \eqref{eq:vertical_weights_index_set}:
\[
\hat\omega_{j,k}(t)
=
\hat\omega_{j,k}
\left(
\hat h_T(t),\hat h_\Omega(t);[n+1]
\right).
\]
If more than one bandwidth pair minimizes
\eqref{eq:augmented_bandwidth_selection}, we select the pair with the smallest
index under a prespecified ordering of
\(\mathcal H_T\times\mathcal H_\Omega\). For example, we order pairs first by
increasing \(h_T\) and then by increasing \(h_\Omega\). This deterministic
tie-breaking rule does not depend on the observation labels.
Because the criterion in \eqref{eq:augmented_bandwidth_objective} treats all
\(n+1\) augmented observations symmetrically, this bandwidth-selection rule is
permutation invariant. Consequently, the resulting residuals are
exchangeable under the assumptions of \Cref{thm:hv_marginal_validity}, and the marginal validity
continues to hold.

\subsection{Efficient online implementation}
\label{app:efficient_online_bandwidth}

Although \eqref{eq:augmented_bandwidth_selection} appears to require recomputing all
leave-one-out predictions for every candidate \(t\), its candidate dependence can be
updated exactly at low cost.

Denote the finite grid of bandwidth pairs by
\(
\mathcal G
:=
\mathcal H_T\times\mathcal H_\Omega.
\)
For conciseness, write
\[
g=(h_T,h_\Omega)\in\mathcal G,
\qquad
\ell_{jk}(g):=\ell_{jk}(h_T,h_\Omega),
\qquad
K_{jk}(g):=\exp\{-\ell_{jk}(g)\}.
\]
\paragraph{Calibration-only step.}
For the calibration points, define
\[
e_k^\star
:=
T_k^\star-\hat T_k^H,
\qquad
k\in[n].
\]
For each bandwidth pair \(g\in\mathcal G\) and calibration point \(j\in[n]\),
precompute
\begin{align}
D_j^{(n)}(g)
&:=
\sum_{k\in[n]\setminus\{j\}}
K_{jk}(g),
\label{eq:online_calibration_denominator}\\
N_j^{(n)}(g)
&:=
\sum_{k\in[n]\setminus\{j\}}
K_{jk}(g)e_k^\star,
\label{eq:online_calibration_numerator}
\end{align}
and the calibration-only local residual average
\begin{equation}
\label{eq:online_calibration_residual_average}
\bar e_j^{(n)}(g)
:=
\frac{N_j^{(n)}(g)}{D_j^{(n)}(g)}= 
\sum_{k\in[n]\setminus\{j\}}
\hat\omega_{j,k}(g;[n])e_k^\star, \qquad \hat\omega_{j,k}(g;[n])
=
\frac{K_{jk}(g)}{D_j^{(n)}(g)}.
\end{equation}

\paragraph{Online state update.}
Given  a  state \(S_{n+1}\), compute
\(K_{j,n+1}(g)\) for \(j\in[n]\). Define
\begin{equation}
\label{eq:test_kernel_weight_pi}
\pi_j(g)
:=
\frac{
K_{j,n+1}(g)
}{
D_j^{(n)}(g)+K_{j,n+1}(g)
}.
\end{equation}
For each \(j\in[n]\), the augmented weights satisfy
\begin{equation}
\label{eq:online_weight_update}
\hat\omega_{j,k}(g;[n+1])
=
\begin{cases}
\{1-\pi_j(g)\}\hat\omega_{j,k}(g;[n]),
&
k\in[n]\setminus\{j\},
\\[3pt]
\pi_j(g),
&
k=n+1.
\end{cases}
\end{equation}
Thus, adding the test point multiplies the \(n-1\) calibration-only weights in row
\(j\) by the common factor \(1-\pi_j(g)\). The test point receives weight
\(\pi_j(g)\), and the updated weights $\hat\omega_{j,k}(g;[n+1])$, $k\in[n+1]\setminus\{j\}$, continue to sum to one.

\paragraph{Online stopping-time update.}

For a candidate
\(t\in\{t_0,\ldots,t_{\max}\}\), define the augmented horizontal residual of the
test point by
\begin{equation}
\label{eq:test_horizontal_residual_candidate}
e_{n+1}(t)
:=
T_{n+1}^\star(t)-\hat T_{n+1}^H
=
t-\hat T_{n+1}^H,
\end{equation}
by the definition of the augmented labels.

For each calibration point \(j\in[n]\), its augmented local residual
average is
\begin{align}
\bar e_j(t;g)
&:=
\sum_{k\in[n]\setminus\{j\}}
\hat\omega_{j,k}(g;[n+1])e_k^\star
+
\hat\omega_{j,n+1}(g;[n+1])e_{n+1}(t)
\nonumber\\
&=
\frac{
N_j^{(n)}(g)
+
K_{j,n+1}(g)e_{n+1}(t)
}{
D_j^{(n)}(g)
+
K_{j,n+1}(g)
}
\nonumber\\
&=
\{1-\pi_j(g)\}\bar e_j^{(n)}(g)
+
\pi_j(g)e_{n+1}(t).
\label{eq:online_local_residual_average}
\end{align}
The dependence on the candidate
\(t\) enters only through \(e_{n+1}(t)\).

For the test point, the leave-one-out local residual average is
\begin{align}
\bar e_{n+1}(g)
&:=
\sum_{k=1}^n
\hat\omega_{n+1,k}(g;[n+1])e_k^\star
=
\frac{
\displaystyle
\sum_{k=1}^n
K_{n+1,k}(g)e_k^\star
}{
\displaystyle
\sum_{k=1}^n
K_{n+1,k}(g)
}.
\label{eq:test_local_residual_average}
\end{align}
Unlike \(\bar e_j(t;g)\) for \(j\in[n]\), the quantity
\(\bar e_{n+1}(g)\) does not depend on \(t\), because observation \(n+1\) is
excluded from its own leave-one-out correction.

Therefore, the residual scores used to compute the conformal \(p\)-value are
\begin{equation}
\label{eq:online_residual_scores}
R_j(t;g)
=
\begin{cases}
e_j^\star-\bar e_j(t;g),
&
j\in[n],
\\[5pt]
e_{n+1}(t)-\bar e_{n+1}(g),
&
j=n+1.
\end{cases}
\end{equation}

\paragraph{Quadratic objective.}

Substituting \eqref{eq:online_residual_scores} into
\eqref{eq:augmented_bandwidth_objective} gives
\begin{align}
Q_g(t)
&:=
Q_{h_T,h_\Omega}(t)
\nonumber\\
&=
\frac{1}{n+1}
\sum_{j=1}^n
\left[
e_j^\star
-
\{1-\pi_j(g)\}\bar e_j^{(n)}(g)
-
\pi_j(g)e_{n+1}(t)
\right]^2
\nonumber\\
&\quad
+
\frac{1}{n+1}
\left\{
e_{n+1}(t)-\bar e_{n+1}(g)
\right\}^2
\nonumber\\
&=
A_g e_{n+1}(t)^2
-
2B_g e_{n+1}(t)
+
C_g,
\label{eq:quadratic_bandwidth_objective}
\end{align}
where
\begin{equation}
\label{eq:quadratic_bandwidth_coefficients}
\begin{split}
A_g
&=
\frac{
1+\sum_{j=1}^n\pi_j(g)^2
}{
n+1
},
\\[4pt]
B_g
&=
\frac{1}{n+1}
\left[
\bar e_{n+1}(g)
+
\sum_{j=1}^n
\pi_j(g)
\left\{
e_j^\star
-
\{1-\pi_j(g)\}\bar e_j^{(n)}(g)
\right\}
\right],
\\[4pt]
C_g
&=
\frac{1}{n+1}
\left[
\bar e_{n+1}(g)^2
+
\sum_{j=1}^n
\left\{
e_j^\star
-
\{1-\pi_j(g)\}\bar e_j^{(n)}(g)
\right\}^2
\right].
\end{split}
\end{equation}

The coefficients \(A_g\), \(B_g\), and \(C_g\) are computed once for each
bandwidth pair after the test state is observed. For candidate \(t\), the exact
candidate-specific bandwidth pair is
\[
\hat g(t)
=
\bigl(
\hat h_T(t),\hat h_\Omega(t)
\bigr)
\in
\arg\min_{g\in\mathcal G}
Q_g(t).
\]
If the minimum is attained by multiple bandwidth pairs, ties are resolved using a
deterministic ordering of \(\mathcal G\) fixed in advance and independent of the
observation labels.

\begin{algorithm}[t]
\caption{Bandwidth selection for full conformal prediction}
\label{alg:online_full_conformal_bandwidth}
\small
\begin{algorithmic}[1]
\algrenewcommand\algorithmicrequire{\textbf{Input:}}
\algrenewcommand\algorithmicensure{\textbf{Output:}}

\Require Bandwidth grid \(\mathcal G\), calibration data
\(\{S_j,T_j^\star,\hat T_j^H\}_{j=1}^n\), test data
\((S_{n+1},\hat T_{n+1}^H)\), and miscoverage level \(\alpha\).

\Ensure Conformalized stopping time \(\hat U_{n+1}^{HV}\).

\State Compute \(e_j^\star\gets T_j^\star-\hat T_j^H\) for all \(j\in[n]\).
\State \(e\gets t_0-\hat T_{n+1}^H\).
\State \(\hat U_{n+1}^{HV}\gets t_{\max}\).

\ForAll{\(g\in\mathcal G\)}
    \Comment{\textbf{Calibration-only step}}
    \State Compute
    \(\{D_j^{(n)}(g),N_j^{(n)}(g),\bar e_j^{(n)}(g)\}_{j=1}^n\)
    using \eqref{eq:online_calibration_denominator}--%
    \eqref{eq:online_calibration_residual_average}.
\EndFor

\ForAll{\(g\in\mathcal G\)}
    \Comment{\textbf{Online state update}}
    \State Compute
    \(\{\pi_j(g)\}_{j=1}^n\) and \(\bar e_{n+1}(g)\)
    using \eqref{eq:test_kernel_weight_pi} and
    \eqref{eq:test_local_residual_average}.
    \State Compute \(A_g,B_g,C_g\) using
    \eqref{eq:quadratic_bandwidth_coefficients}, and set
    \(Q_g\gets A_g e^2-2B_g e+C_g\).
\EndFor

\For{\(t=t_0,\ldots,t_{\max}\)}
    \Comment{\textbf{Online stopping-time update}}
    \State
    \(\displaystyle
    \hat g(t)\gets
    \argmin_{g\in\mathcal G}Q_g
    \).
    \State Compute the $p$-value \(p^{HV}(t)\) using
    \(\{R_j(t;\hat g(t))\}_{j=1}^{n+1}\) from
     \eqref{eq:online_residual_scores} .

    \If{\(p^{HV}(t)>\alpha\)}
        \State \(\hat U_{n+1}^{HV}\gets t\).
    \EndIf

    \If{\(t<t_{\max}\)}
        \ForAll{\(g\in\mathcal G\)}
            \State \(Q_g\gets Q_g+A_g(2e+1)-2B_g\).
        \EndFor
        \State \(e\gets e+1\).
    \EndIf
\EndFor

\State \Return \(\hat U_{n+1}^{HV}\).

\end{algorithmic}
\end{algorithm}

For consecutive integer candidates,
\(
e_{n+1}(t+1)
=
e_{n+1}(t)+1.
\)
Therefore,
\begin{equation}
\label{eq:recursive_bandwidth_objective}
Q_g(t+1)
=
Q_g(t)
+
A_g\{2e_{n+1}(t)+1\}
-
2B_g.
\end{equation}
Thus, after \(Q_g(t_0)\) is initialized, each subsequent objective value can be
computed in constant time for a fixed bandwidth pair.

\paragraph{Computational complexity.}
Let
\(
G
=
|\mathcal G|
=
|\mathcal H_T|\,|\mathcal H_\Omega|
\)
be the number of bandwidth pairs, and let
\(
N_{\mathrm{cand}}
=
t_{\max}-t_0+1
\)
be the number of candidate stopping times. The calibration-only precomputation,
including the pairwise kernel values and their sums, costs
\(O(Gn^2)\). This step is performed once and reused for all test images.

For each new test image, computing its kernel similarities with the \(n\) calibration
states and constructing the quadratic coefficients costs \(O(Gn)\). Updating the
bandwidth-selection objectives and minimizing over the \(G\) bandwidth pairs for all
\(t\in\{t_0,\ldots,t_{\max}\}\) costs \(O(GN_{\mathrm{cand}})\).
Computing the \(n+1\) residual scores and the conformal \(p\)-value for every
candidate costs \(O(nN_{\mathrm{cand}})\). Therefore, after the one-time
calibration precomputation, the cost of computing the upper bound for each new test
image is
\[
O\left(
Gn+GN_{\mathrm{cand}}+nN_{\mathrm{cand}}
\right).
\]
For fixed bandwidth grids and a fixed candidate range, the per-test-image cost is
linear in the calibration sample size \(n\). In our experiments, we used \(G=81\), while \(N_{\mathrm{cand}}\) is on the order of tens.

\subsection{Calibration-only approximation}
\label{app:calibration_only_approximation}

A simpler implementation selects the bandwidth pair once using only the \(n\)
calibration observations. For a fixed bandwidth pair
\(g=(h_T,h_\Omega)\in\mathcal G\), define the calibration-only leave-one-out
horizontal--vertical prediction
\[
\hat T_j^{HV,\mathrm{cal}}(g)
=
\hat T_j^H
+
\sum_{k\in[n]\setminus\{j\}}
\hat\omega_{j,k}(g;[n])e_k^\star
=
\hat T_j^H+\bar e_j^{(n)}(g),
\qquad
j\in[n],
\]
where
\(
e_k^\star=T_k^\star-\hat T_k^H
\)
and \(\bar e_j^{(n)}(g)\) is defined in
\eqref{eq:online_calibration_residual_average}. The calibration-only
objective is
\[
Q_g^{\mathrm{cal}}
:=
\frac1n
\sum_{j=1}^n
\left\{
T_j^\star-\hat T_j^{HV,\mathrm{cal}}(g)
\right\}^2
=
\frac1n
\sum_{j=1}^n
\left\{
e_j^\star-\bar e_j^{(n)}(g)
\right\}^2.
\]

Using the same prespecified ordering of \(\mathcal G\) as in the exact
procedure, we select $\tilde g
=
(\tilde h_T,\tilde h_\Omega)
\in
\arg\min_{g\in\mathcal G}
Q_g^{\mathrm{cal}}.$
The selected bandwidth pair \(\tilde g\) is then held fixed for every candidate value
\(t\in\{t_0,\ldots,t_{\max}\}\). The weights used in the residual scores are
computed over all \(n+1\) states:
\[
\tilde\omega_{j,k}
:=
\hat\omega_{j,k}(\tilde g;[n+1]),
\qquad
k\in[n+1]\setminus\{j\}.
\]
Thus, the test state participates in the final local correction, but its hypothetical
stopping-time label does not participate in bandwidth selection. Let
\(\tilde U_{n+1}^{HV}\) denote the resulting stopping rule.

This procedure does not alter the exchangeability of the observations themselves.
However, its score-construction process is not permutation invariant, because the first
\(n\) stopping-time labels are used to select \(\tilde g\), whereas the hypothetical
value $t$ of the stopping time $T_{n+1}$ is not. Consequently, the exact finite-sample rank
argument in \Cref{thm:hv_marginal_validity} does not directly apply to the
calibration-only approximation.

\paragraph{One-point stability.}
The difference between the calibration-only and exact augmented procedures can be
quantified directly from the kernel weights. For a fixed bandwidth pair
\(g\in\mathcal G\), define the calibration-only residual score
\[
R_j^{\mathrm{cal}}(g)
:=
e_j^\star-\bar e_j^{(n)}(g),
\qquad
j\in[n].
\]
Using \eqref{eq:online_local_residual_average}, the augmented residual score
$R_j(t;g)$ for  $j\in[n]$
satisfies
\begin{align}
R_j(t;g)
=
e_j^\star-\bar e_j(t;g)
=
R_j^{\mathrm{cal}}(g)
+
\pi_j(g)
\left\{
\bar e_j^{(n)}(g)-e_{n+1}(t)
\right\}.
\label{eq:augmented_calibration_residual_difference}
\end{align}
Thus, \(\pi_j(g)\) is the exact influence of the augmented test point on $R_j(t;g)$.

Denote the calibration-only and candidate-specific selected bandwidth pairs by
\[
\tilde g
\in
\arg\min_{g\in\mathcal G}
Q_g^{\mathrm{cal}},
\qquad
\hat g(t)
\in
\arg\min_{g\in\mathcal G}
Q_g(t).
\]
When \(\tilde g\) is the unique calibration-only minimizer, define the calibration
margin
\[
\Delta_n^{\mathrm{cal}}
:=
\min_{g\in\mathcal G\setminus\{\tilde g\}}
\left\{
Q_g^{\mathrm{cal}}
-
Q_{\tilde g}^{\mathrm{cal}}
\right\}.
\]

Define the maximal average influence of the augmented test point by
\[
\bar\pi_n
:=
\max_{g\in\mathcal G}
\frac{1}{n}
\sum_{j=1}^n
\pi_j(g),
\]
where \(\pi_j(g)\), defined in
\eqref{eq:test_kernel_weight_pi}, is the weight assigned to the augmented test point
in the neighborhood of calibration point \(j\). Define the objective-stability
parameter
\[
\varepsilon_n^{\mathrm{stab}}
:=
\frac{
8(t_{\max}-t_0)^2
\left(
n\bar\pi_n+1
\right)
}{
n+1
}.
\]

\begin{proposition}[Stability]
\label{prop:calibration_only_stability}
The augmented and calibration-only objectives satisfy
\begin{equation}
\label{eq:calibration_augmented_objective_difference}
\max_{g\in\mathcal G}
\max_{t\in\{t_0,\ldots,t_{\max}\}}
\left|
Q_g(t)-Q_g^{\mathrm{cal}}
\right|
\le
\varepsilon_n^{\mathrm{stab}}.
\end{equation}
Moreover, for every candidate \(t\),
\[
0
\le
Q_{\hat g(t)}^{\mathrm{cal}}
-
Q_{\tilde g}^{\mathrm{cal}}
\le
2\varepsilon_n^{\mathrm{stab}}.
\]
If \(\tilde g\) is the unique calibration-only minimizer and
\(
\Delta_n^{\mathrm{cal}}
>
2\varepsilon_n^{\mathrm{stab}},
\)
then
\[
\hat g(t)=\tilde g
\qquad
\text{for every }
t\in\{t_0,\ldots,t_{\max}\}.
\]
\end{proposition}
The quantity \(\bar\pi_n\) measures the largest average
weight over the bandwidth grid, assigned to the augmented test point across the \(n\) calibration
neighborhoods. Since \(0\le\pi_j(g)\le1\), we have
\(
0\le\bar\pi_n\le1.
\)
Individual values of \(\pi_j(g)\) can be much larger than the nominal value
\(1/n\). For example, a test
state that is especially close to a particular calibration state may receive
substantial weight in that calibration neighborhood.

If \(S_1,\ldots,S_{n+1}\) are exchangeable, then for each fixed bandwidth
pair \(g\in\mathcal G\), exchangeability provides an average control. Conditional on
\(S_j\), the \(n\) states
\(\{S_k:k\ne j\}\) are exchangeable, and their normalized kernel weights sum to one.
Therefore, $\mathbb E\left\{
\pi_j(g)
\,\middle|\,
S_j
\right\}
=
n^{-1}.$
Thus,
\[
\mathbb E\left[
\frac{1}{n}
\sum_{j=1}^n
\pi_j(g)
\right]
=
\frac{1}{n}.
\]
If the bandwidth grid has cardinality
\(
G=|\mathcal G|,
\)
then
\[
\begin{aligned}
\mathbb E(\bar\pi_n)
=
\mathbb E\left[
\max_{g\in\mathcal G}
\frac{1}{n}
\sum_{j=1}^n
\pi_j(g)
\right] 
\le
\sum_{g\in\mathcal G}
\mathbb E\left[
\frac{1}{n}
\sum_{j=1}^n
\pi_j(g)
\right]
=
\frac{G}{n}.
\end{aligned}
\]
Because the bandwidth grid \(\mathcal G\) has cardinality
\(G=|\mathcal G|\) that does not depend on \(n\), the bound
\(\mathbb E(\bar\pi_n)\le G/n\) implies
\(\bar\pi_n=O_p(1/n)\). Therefore,
\(\varepsilon_n^{\mathrm{stab}}=O_p(1/n)\). If, in addition,
\(
\Pp\left\{
\Delta_n^{\mathrm{cal}}
>
2\varepsilon_n^{\mathrm{stab}}
\right\}
\rightarrow 1,
\)
then \Cref{prop:calibration_only_stability} implies
\[
\Pp\left\{
\tilde U_{n+1}^{HV}
=
\hat U_{n+1}^{HV}
\right\}
\rightarrow 1,
\]
i.e., the two stopping rules are asymptotically
equivalent.

\section{Technical proofs}
\label{app:technical_proofs}

\subsection{Proof of \Cref{prop:next_band_residual_power_law_decay}}
\label{app:proof_next_band_residual_power_law_decay}

\begin{proof}
Under \Cref{ass:ordered_tail_preprojection_residual}, if
\(p_i(\tau)\equiv p_i\), then for every
\(x\in(\omega_{t-1},\omega_t]\),
\[
[\mathcal T E_i(t-1)](x)=A_i x^{-p_i}.
\]
Hence
\[
I_{i,t}^\star
=
K^{-1}A_i^2
\int_{\omega_{t-1}}^{\omega_t}x^{-2p_i}\,dx.
\]
Because \(p_i>1/2\), direct integration gives
\[
I_{i,t}^\star
=
K^{-1}A_i^2
\frac{
\omega_{t-1}^{-(2p_i-1)}-\omega_t^{-(2p_i-1)}
}{
2p_i-1
}.
\]
Using the logarithmic grid
\(\omega_t=\exp\{\Delta(t-1)\}\), we obtain
\[
\omega_{t-1}^{-(2p_i-1)}-\omega_t^{-(2p_i-1)}
=
\exp\{-(2p_i-1)\Delta(t-2)\}
\left[
1-\exp\{-(2p_i-1)\Delta\}
\right].
\]
Therefore,
\[
I_{i,t}^\star
=
K^{-1}A_i^2
\frac{
1-\exp\{-(2p_i-1)\Delta\}
}{
2p_i-1
}
\exp\{-(2p_i-1)\Delta(t-2)\}.
\]
Taking logarithms proves the result.
\end{proof}

\subsection{Proof of \Cref{prop:preprojection_residual_polynomial_trend}}
\label{app:proof_preprojection_residual_polynomial_trend}

\begin{proof}
For an acquisition index \(t\) near \(t_0\),
\Cref{ass:ordered_tail_preprojection_residual} gives
\[
[\mathcal T E_i(t-1)](x)
=
A_i x^{-p_i(t-1)}
\]
for \(x\in(\omega_{t-1},\omega_t]\). Consequently,
\[
I_{i,t}^\star
=
K^{-1}A_i^2
\int_{\omega_{t-1}}^{\omega_t}
x^{-2p_i(t-1)}\,dx.
\]
Since \(p_i(t-1)>1/2\), evaluating the integral and using
\(\omega_t=\exp\{\Delta(t-1)\}\) yields
\[
I_{i,t}^\star
=
K^{-1}A_i^2
\exp\{-[2p_i(t-1)-1]\Delta(t-2)\}
\frac{
1-\exp\{-[2p_i(t-1)-1]\Delta\}
}{
2p_i(t-1)-1
}.
\]

To apply Taylor's theorem, extend the acquisition index to a real variable
\(u\) near \(t_0\), and define
\[
\begin{aligned}
F_i(u)
:={}&
-\log K
+
2\log A_i
-
[2p_i(u-1)-1]\Delta(u-2)
\\
&\quad+
\log\left[
\frac{
1-\exp\{-[2p_i(u-1)-1]\Delta\}
}{
2p_i(u-1)-1
}
\right].
\end{aligned}
\]
For every integer \(t\) near \(t_0\), the preceding calculation gives
\(F_i(t)=\log I_{i,t}^\star\). By assumption, \(p_i\) is \(q+1\) times
continuously differentiable near \(t_0-1\), and
\(2p_i(u-1)-1\) remains strictly positive in a sufficiently small
neighborhood of \(t_0\). Hence \(F_i\) is \(q+1\) times continuously
differentiable there. Taylor's theorem around \(t_0\) therefore gives
coefficients
\(\gamma_{i,0}^{(q)},\ldots,\gamma_{i,q}^{(q)}\) such that
\[
F_i(t)
=
\sum_{\ell=0}^q
\gamma_{i,\ell}^{(q)}(t-t_0)^\ell
+
O_q\!\left(|t-t_0|^{q+1}\right).
\]
Substituting \(F_i(t)=\log I_{i,t}^\star\) proves the first claim.

It remains to establish monotonicity. Since
\(\omega_t=e^\Delta\omega_{t-1}\) and
\(\omega_{t+1}=e^\Delta\omega_t\), the change of variables
\(x=e^\Delta y\) gives
\[
\begin{aligned}
I_{i,t+1}^\star
&=
K^{-1}A_i^2
\int_{\omega_t}^{\omega_{t+1}}
x^{-2p_i(t)}\,dx
\\
&=
K^{-1}A_i^2
e^{-[2p_i(t)-1]\Delta}
\int_{\omega_{t-1}}^{\omega_t}
y^{-2p_i(t)}\,dy.
\end{aligned}
\]
If \(p_i(t)\ge p_i(t-1)\), then \(y\ge1\) on the integration interval implies
\[
y^{-2p_i(t)}
\le
y^{-2p_i(t-1)},
\]
while
\(
e^{-[2p_i(t)-1]\Delta}\le1.
\)
Therefore,
\[
I_{i,t+1}^\star
\le
K^{-1}A_i^2
\int_{\omega_{t-1}}^{\omega_t}
y^{-2p_i(t-1)}\,dy
=
I_{i,t}^\star.
\]
This proves the second claim.
\end{proof}

\subsection{Proof of \Cref{prop:monotone_projection_gain}}
\label{app:proof_monotone_projection_gain}

\begin{proof}
Fix image \(i\), and write \(r=\hat q_i\). Let \(X_r\) be the design matrix
whose row indexed by \(t\in\mathcal R_0\) is \(\phi_r(t)^\top\). Then
\[
\|\gamma\|_{\mathcal R_0,r}^2
=
\|X_r\gamma\|_2^2.
\]
Because the design on \(\mathcal R_0\) has full column rank, this is a norm on
\(\Gamma_r\).

The first inequality follows immediately from the definition of
\(C_{\mathrm{ext}}(r)\). Applying that definition to
\(\hat\gamma_{i,r}-\gamma_i^\star\) gives
\[
\|\hat\gamma_{i,r}-\gamma_i^\star\|_{\mathcal R_0^{\mathrm c},r}^2
\le
C_{\mathrm{ext}}(r)
\|\hat\gamma_{i,r}-\gamma_i^\star\|_{\mathcal R_0,r}^2.
\]

We next consider the monotone fit. Let
\(
H_i=(H_{i,t}:t\in\mathcal R_0)
\).
The unconstrained estimator \(\hat\gamma_{i,r}\) is the unique minimizer of
\[
\|H_i-X_r\gamma\|_2^2
\]
over \(\Gamma_r\), and its normal equations give
\[
X_r^\top(H_i-X_r\hat\gamma_{i,r})=0.
\]
Thus, for every \(\gamma\in\Gamma_r\),
\[
\|H_i-X_r\gamma\|_2^2
=
\|H_i-X_r\hat\gamma_{i,r}\|_2^2
+
\|\gamma-\hat\gamma_{i,r}\|_{\mathcal R_0,r}^2.
\]
Therefore, minimizing the least-squares criterion over
\(\Gamma_r^{\mathrm{mon}}\) is equivalent to projecting
\(\hat\gamma_{i,r}\) onto \(\Gamma_r^{\mathrm{mon}}\) under the norm
\(\|\cdot\|_{\mathcal R_0,r}\). The set
\(\Gamma_r^{\mathrm{mon}}\) is closed and convex because it is defined by
finitely many linear inequalities.

By assumption,
\(\gamma_i^\star\in\Gamma_r^{\mathrm{mon}}\). The Pythagorean inequality
for projection onto a closed convex set therefore gives
\[
\|\hat\gamma_{i,r}^{\mathrm{mon}}-\gamma_i^\star\|_{\mathcal R_0,r}^2
\le
\|\hat\gamma_{i,r}-\gamma_i^\star\|_{\mathcal R_0,r}^2
-
\|\hat\gamma_{i,r}-\hat\gamma_{i,r}^{\mathrm{mon}}\|_{\mathcal R_0,r}^2.
\]
Applying the definition of \(C_{\mathrm{ext}}(r)\) to
\(\hat\gamma_{i,r}^{\mathrm{mon}}-\gamma_i^\star\) gives
\[
\|\hat\gamma_{i,r}^{\mathrm{mon}}-\gamma_i^\star\|_{\mathcal R_0^{\mathrm c},r}^2
\le
C_{\mathrm{ext}}(r)
\|\hat\gamma_{i,r}^{\mathrm{mon}}-\gamma_i^\star\|_{\mathcal R_0,r}^2.
\]
Combining the last two inequalities and substituting
\(r=\hat q_i\) proves the result.
\end{proof}

\subsection{Proof of \Cref{thm:hv_marginal_validity}}
\label{app:proof_hv_marginal_validity}

\begin{proof}
Let
\(
t^\star:=T_{n+1}^\star=\max\{T_{n+1},t_0\}.
\)
Because every deployed stopping rule takes values in
\(\{t_0,\ldots,t_{\max}\}\), the events $\{T_{n+1}^\star\le\hat U_{n+1}^{HV}\}$ and $\{T_{n+1}\le\hat U_{n+1}^{HV}\}$
coincide. Indeed, this is immediate when \(T_{n+1}\ge t_0\), while both
events hold when \(T_{n+1}<t_0\).

At the candidate \(t=t^\star\), the augmented labels equal the actual floored
stopping times:
\[
T_j^\star(t^\star)=T_j^\star,
\qquad
j=1,\ldots,n+1.
\]
Hence the augmented sample at the true candidate is
\[
\bigl(
(S_1,T_1^\star),\ldots,(S_{n+1},T_{n+1}^\star)
\bigr).
\]
The images are exchangeable, and the state, horizontal prediction, and floored
stopping time are constructed from each image by the same rule, using a
reconstruction model fitted on separate training data. Therefore, the
pairs \((S_j,T_j^\star)\), \(j\in[n+1]\), are exchangeable.

The bandwidth selector is permutation invariant by \eqref{equ:sym}. For any
selected bandwidth, permuting the augmented observations permutes the pairwise
distances, the rows and columns of the weight matrix, and hence the
leave-one-out predictions and residual scores in the same way. Thus
\(
\bigl(
R_1(t^\star),\ldots,R_{n+1}(t^\star)
\bigr)
\)
is exchangeable.

For a deterministic score vector \(z=(z_1,\ldots,z_{n+1})\), define
\[
\rho_i(z)
:=
\frac{1}{n+1}
\sum_{j=1}^{n+1}
\mathbbm{1}\{z_j\ge z_i\}.
\]
For any \(u\in[0,1]\), at most
\(\lfloor u(n+1)\rfloor\) indices can satisfy \(\rho_i(z)\le u\).
Therefore, exchangeability implies
\[
\begin{aligned}
\Pp\{p^{HV}(t^\star)\le u\}
=
\mathbb E\left[
\frac{1}{n+1}
\sum_{i=1}^{n+1}
\mathbbm{1}\{\rho_i(R(t^\star))\le u\}
\right]\le u.
\end{aligned}
\]
Thus \(p^{HV}(t^\star)\) is super-uniform, including in the presence of ties, and
\[
\Pp\{p^{HV}(t^\star)>\alpha\}
\ge
1-\alpha.
\]

Whenever \(p^{HV}(t^\star)>\alpha\), the true floored stopping time belongs to
the retained set in \eqref{equ:u_alpha}. Since \(\hat U_{n+1}^{HV}\) is the
largest retained candidate,
\[
\hat U_{n+1}^{HV}
\ge
t^\star
\ge
T_{n+1}.
\]
Consequently,
\[
\Pp\{T_{n+1}\le\hat U_{n+1}^{HV}\}
\ge
1-\alpha.
\]

Finally, if \(L_{n+1}(t)\) is nonincreasing, then on the event
\(T_{n+1}\le\hat U_{n+1}^{HV}\),
\[
L_{n+1}(\hat U_{n+1}^{HV})
\le
L_{n+1}(T_{n+1})
\le
c.
\]
This proves the reconstruction-error guarantee.
\end{proof}

\subsection{Proof of \Cref{thm:hv_approx_conditional_coverage}}
\label{app:proof_hv_approx_conditional_coverage}

\begin{proof}
Throughout the proof, the bandwidth pair
\(g=(h_T,h_\Omega)\) is fixed. All weights, scores, \(p\)-values, and stopping
rules refer to this fixed-bandwidth construction, and we suppress \(g\) from
the notation.

Let
\(
t^\star:=T_{n+1}^\star.
\)
If \(p^{HV}(t^\star)>\alpha\), then \(t^\star\) is retained and
\[
\hat U_{n+1}^{HV}(g)
\ge
t^\star
\ge
T_{n+1}.
\]
Therefore,
\[
\Pp\left\{
T_{n+1}>\hat U_{n+1}^{HV}(g)
\,\middle|\,
S_{n+1}
\right\}
\le
\Pp\left\{
p^{HV}(t^\star)\le\alpha
\,\middle|\,
S_{n+1}
\right\}.
\]
It remains to bound the probability on the right.

At \(t=t^\star\), the augmented labels equal the actual floored stopping times,
so \(e_j(t^\star)=e_j^\star\) for every \(j\). By
\Cref{ass:feature_bias_model},
\[
e_j^\star=m_n(S_j)+\xi_j.
\]
Substituting this decomposition into \eqref{eq:hv_residual_score} gives
\[
R_j(t^\star)
=
\xi_j+B_j-Z_j,
\]
where
\[
B_j
:=
m_n(S_j)
-
\sum_{k\ne j}\hat\omega_{j,k}m_n(S_k),
\qquad
Z_j
:=
\sum_{k\ne j}\hat\omega_{j,k}\xi_k.
\]

By the reproducing property and the Cauchy--Schwarz inequality,
\[
\begin{aligned}
|B_j|
&=
\left|
\left\langle
m_n,
\kappa(S_j,\cdot)
-
\sum_{k\ne j}
\hat\omega_{j,k}\kappa(S_k,\cdot)
\right\rangle_{\mathcal H_\kappa}
\right|
\\
&\le
\|m_n\|_{\mathcal H_\kappa}
\left\|
\kappa(S_j,\cdot)
-
\sum_{k\ne j}
\hat\omega_{j,k}\kappa(S_k,\cdot)
\right\|_{\mathcal H_\kappa}
\\
&\le
\Lambda_{\kappa,n}D_{\kappa,n}.
\end{aligned}
\]

We next control \(Z_j\). Let
\[
K_{jk}(g):=\exp\{-\ell_{jk}(g)\}.
\]
Since
\(\hat T_i^H\in[t_0,t_{\max}]\) and \(\hat\Omega_i\in[0,1]\),
\[
\ell_{jk}(g)
\le
\frac{(t_{\max}-t_0)^2}{h_T^2}
+
\frac{1}{h_\Omega^2}.
\]
Define
\[
c_g
:=
\exp\left\{
-\frac{(t_{\max}-t_0)^2}{h_T^2}
-\frac{1}{h_\Omega^2}
\right\}>0.
\]
Then \(c_g\le K_{jk}(g)\le1\). Each leave-one-out denominator contains
exactly \(n\) terms, so
\[
\sum_{\ell\ne j}K_{j\ell}(g)\ge nc_g,
\qquad
\hat\omega_{j,k}\le\frac{1}{nc_g}.
\]
Because the weights are nonnegative and sum to one,
\[
\sum_{k\ne j}\hat\omega_{j,k}^2
\le
\left(\max_{k\ne j}\hat\omega_{j,k}\right)
\sum_{k\ne j}\hat\omega_{j,k}
\le
\frac{1}{nc_g}.
\]

Conditional on \(S_1,\ldots,S_{n+1}\), the weights are fixed. Independence and
sub-Gaussianity of the \(\xi_k\)'s therefore imply that, for every
\(\lambda\in\mathbb R\),
\[
\begin{aligned}
\mathbb E\left[
\exp(\lambda Z_j)
\,\middle|\,
S_1,\ldots,S_{n+1}
\right]
&\le
\exp\left\{
\frac{\sigma^2\lambda^2}{2}
\sum_{k\ne j}\hat\omega_{j,k}^2
\right\}
\le
\exp\left\{
\frac{\sigma^2\lambda^2}{2nc_g}
\right\}.
\end{aligned}
\]
Thus, conditionally on the states, \(Z_j\) is
\(\sigma^2/(nc_g)\)-sub-Gaussian. Define
\[
b_{n,g}
:=
\frac{\sigma}{\sqrt{nc_g}}
\sqrt{2\log\{2(n+1)n^2\}}.
\]
A union bound over \(j\in[n+1]\) shows that
$\mathcal E_n
:=
\left\{
\max_{1\le j\le n+1}|Z_j|\le b_{n,g}
\right\}$
satisfies
\[
\Pp(\mathcal E_n^c\mid S_1,\ldots,S_{n+1})
\le
\frac{1}{n^2}.
\]
On \(\mathcal E_n\),
\[
\max_{1\le j\le n+1}
|R_j(t^\star)-\xi_j|
\le
r_{n,g},
\qquad
r_{n,g}
:=
\Lambda_{\kappa,n}D_{\kappa,n}+b_{n,g}.
\]
The quantity \(r_{n,g}\) is measurable with respect to the states.

We first derive a comparison that does not require a density and applies to
discrete, continuous, or mixed residual distributions. Define the concentration
function
\[
d_\xi(r)
:=
\sup_{z\in\mathbb R}
\Pp\{|\xi_1-z|\le r\},
\qquad
r\ge0,
\]
and the ideal rank \(p\)-value
\[
p(\xi)
:=
\frac{1}{n+1}
\sum_{j=1}^{n+1}
\mathbbm{1}\{\xi_j\ge\xi_{n+1}\}.
\]
Conditional on the states, the variables
\(\xi_1,\ldots,\xi_{n+1}\) are i.i.d. Hence, by the same tie-safe rank
argument used in the proof of \Cref{thm:hv_marginal_validity},
\[
\Pp\left\{
p(\xi)\le u
\,\middle|\,
S_1,\ldots,S_{n+1}
\right\}
\le u,
\qquad
u\in[0,1].
\]

On \(\mathcal E_n\), if
\[
\xi_j\ge\xi_{n+1}
\quad\text{and}\quad
|\xi_j-\xi_{n+1}|>2r_{n,g},
\]
then
\[
R_j(t^\star)
\ge
\xi_j-r_{n,g}
>
\xi_{n+1}+r_{n,g}
\ge
R_{n+1}(t^\star).
\]
Therefore,
\(
p^{HV}(t^\star)
\ge
p(\xi)-A_n,
\)
where
\[
A_n
:=
\frac{1}{n+1}
\sum_{j=1}^n
\mathbbm{1}
\left\{
|\xi_j-\xi_{n+1}|
\le
2r_{n,g}
\right\}.
\]

Conditional on the states and on \(\xi_{n+1}\), the indicators in \(A_n\)
are independent, and each has conditional expectation at most
\(d_\xi(2r_{n,g})\). Hoeffding's inequality therefore implies that,
with conditional probability at least \(1-n^{-2}\),
\[
\frac{1}{n}
\sum_{j=1}^n
\mathbbm{1}
\left\{
|\xi_j-\xi_{n+1}|
\le
2r_{n,g}
\right\}
\le
d_\xi(2r_{n,g})
+
\sqrt{\frac{\log n}{n}}.
\]
Let \(\mathcal H_n\) denote this event. Since \(n/(n+1)\le1\), on
\(\mathcal H_n\),
\[
A_n
\le
d_\xi(2r_{n,g})
+
\sqrt{\frac{\log n}{n}}.
\]

On \(\mathcal E_n\cap\mathcal H_n\), the event
\(p^{HV}(t^\star)\le\alpha\) implies
\[
p(\xi)
\le
\alpha
+
d_\xi(2r_{n,g})
+
\sqrt{\frac{\log n}{n}}.
\]
Using the super-uniformity of \(p(\xi)\) and adding the two failure
probabilities gives
\[
\begin{aligned}
&
\Pp\left\{
p^{HV}(t^\star)\le\alpha
\,\middle|\,
S_1,\ldots,S_{n+1}
\right\}
\le
\alpha
+
d_\xi(2r_{n,g})
+
\sqrt{\frac{\log n}{n}}
+
\frac{2}{n^2}.
\end{aligned}
\]
If the threshold on the right of the ideal \(p\)-value exceeds one, the same
inequality remains true because probabilities are bounded by one.

Under the bounded-density condition in
\Cref{ass:feature_bias_model},
\[
d_\xi(2r_{n,g})
\le
4Mr_{n,g}
=
4M\Lambda_{\kappa,n}D_{\kappa,n}
+
4Mb_{n,g}.
\]
Hence
\[
\begin{aligned}
&
\Pp\left\{
p^{HV}(t^\star)\le\alpha
\,\middle|\,
S_1,\ldots,S_{n+1}
\right\}
\le
\alpha
+
4M\Lambda_{\kappa,n}D_{\kappa,n}
+
4Mb_{n,g}
+
\sqrt{\frac{\log n}{n}}
+
\frac{2}{n^2}.
\end{aligned}
\]
Taking conditional expectation over
\(S_1,\ldots,S_n\) given \(S_{n+1}\) yields
\[
\begin{aligned}
&
\Pp\left\{
p^{HV}(t^\star)\le\alpha
\,\middle|\,
S_{n+1}
\right\}
\le
\alpha
+
4M\Lambda_{\kappa,n}
\mathbb E\left[
D_{\kappa,n}
\,\middle|\,
S_{n+1}
\right]
+
4Mb_{n,g}
+
\sqrt{\frac{\log n}{n}}
+
\frac{2}{n^2}.
\end{aligned}
\]

For \(n\ge2\), we have $b_{n,g} \le \eta_g' \sqrt{(\log n)/n}$ for a constant
\(\eta_g'>0\) depending only on \(g,\sigma,t_0\), and \(t_{\max}\). Also,
\[
\frac{2}{n^2}
=
O\left(
\sqrt{\frac{\log n}{n}}
\right).
\]
Therefore, there is a constant \(\eta_g>0\), depending only on
\(g,\sigma,M,t_0\), and \(t_{\max}\), such that
\[
4Mb_{n,g}
+
\sqrt{\frac{\log n}{n}}
+
\frac{2}{n^2}
\le
\eta_g
\sqrt{\frac{\log n}{n}}.
\]
Combining this bound with the initial reduction gives
\[
\Pp\left\{
T_{n+1}>
\hat U_{n+1}^{HV}(g)
\,\middle|\,
S_{n+1}
\right\}
\le
\alpha
+
4M\Lambda_{\kappa,n}
\mathbb E\left[
D_{\kappa,n}
\,\middle|\,
S_{n+1}
\right]
+
\eta_g
\sqrt{\frac{\log n}{n}}.
\]

Finally, monotonicity of the loss implies
\[
\{T_{n+1}\le\hat U_{n+1}^{HV}(g)\}
\subseteq
\{L_{n+1}(\hat U_{n+1}^{HV}(g))\le c\}.
\]
Taking conditional probabilities proves the theorem.
\end{proof}

\paragraph{Remark on grid-valued residuals.}
Before invoking the bounded-density condition, the proof establishes the more
general bound
\[
\begin{aligned}
\Pp\left\{
p^{HV}(t^\star)\le\alpha
\,\middle|\,
S_1,\ldots,S_{n+1}
\right\}
\le
\alpha
+
d_\xi(2r_{n,g})
+
\sqrt{\frac{\log n}{n}}
+
\frac{2}{n^2}.
\end{aligned}
\]
This bound remains valid for discrete \(\xi_j\) when the additive model holds
and the error terms remain i.i.d., sub-Gaussian, and independent of the states,
without requiring the bounded-density condition.
The term
\(d_\xi(2r_{n,g})\) is the largest probability that the unexplained
residual falls within distance \(2r_{n,g}\) of any fixed value. For a discrete
distribution, it includes the probability of exact ties. The bounded-density
assumption is used only to replace this near-tie probability by the simpler
upper bound \(4Mr_{n,g}\).

\subsection{Proof of \Cref{prop:calibration_only_stability}}
\label{app:proof_calibration_only_stability}

\begin{proof}
Write \(M_T:=t_{\max}-t_0\). By construction,
\(T_k^\star,\hat T_k^H,t\in[t_0,t_{\max}]\), so
\[
|e_k^\star|\le M_T,
\qquad
k\in[n],\qquad |e_{n+1}(t)|\le M_T.
\]
Every local residual average is a convex combination of these quantities.
Therefore,
\[
|\bar e_j^{(n)}(g)|\le M_T,
\qquad
|\bar e_j(t;g)|\le M_T,
\qquad
|\bar e_{n+1}(g)|\le M_T.
\]

For \(j\in[n]\), \eqref{eq:online_local_residual_average} gives
\[
\begin{aligned}
R_j(t;g)-R_j^{\mathrm{cal}}(g)
=
\left[e_j^\star-\bar e_j(t;g)\right]
-
\left[e_j^\star-\bar e_j^{(n)}(g)\right]
=
\pi_j(g)
\left\{
\bar e_j^{(n)}(g)-e_{n+1}(t)
\right\}.
\end{aligned}
\]
Hence
\[
|R_j(t;g)-R_j^{\mathrm{cal}}(g)|
\le
2M_T\pi_j(g).
\]
Moreover,
\[
|R_j(t;g)|\le2M_T,
\qquad
|R_j^{\mathrm{cal}}(g)|\le2M_T.
\]
Using \(|a^2-b^2|=|a-b||a+b|\), we obtain
\[
|R_j(t;g)^2-R_j^{\mathrm{cal}}(g)^2|
\le
8M_T^2\pi_j(g).
\]
For the test point,
\[
|R_{n+1}(t;g)|\le2M_T,
\qquad
R_{n+1}(t;g)^2\le4M_T^2,\qquad
\frac{1}{n}
\sum_{j=1}^n
R_j^{\mathrm{cal}}(g)^2
\le
4M_T^2.
\]

Using the definitions of \(Q_g(t)\) and \(Q_g^{\mathrm{cal}}\),
\[
\begin{aligned}
Q_g(t)-Q_g^{\mathrm{cal}}
&=
\frac{1}{n+1}
\sum_{j=1}^n
\left\{
R_j(t;g)^2-R_j^{\mathrm{cal}}(g)^2
\right\}
\\
&\quad+
\frac{R_{n+1}(t;g)^2}{n+1}
+
\left(
\frac{1}{n+1}-\frac{1}{n}
\right)
\sum_{j=1}^n
R_j^{\mathrm{cal}}(g)^2.
\end{aligned}
\]
Since
\(
\left|1/(n+1) -1/n
\right|
=1/[n(n+1)],
\)
the preceding bounds imply
\[
|Q_g(t)-Q_g^{\mathrm{cal}}|
\le
\frac{8M_T^2}{n+1}
\left\{
1+\sum_{j=1}^n\pi_j(g)
\right\}.
\]
Taking the maximum over
\(g\in\mathcal G\) and \(t\in\{t_0,\ldots,t_{\max}\}\), and using
\[
\max_{g\in\mathcal G}
\sum_{j=1}^n\pi_j(g)
=
n\bar\pi_n,
\]
gives
\[
\max_{g\in\mathcal G}
\max_{t\in\{t_0,\ldots,t_{\max}\}}
|Q_g(t)-Q_g^{\mathrm{cal}}|
\le
\frac{8M_T^2(1+n\bar\pi_n)}{n+1}
=
\varepsilon_n^{\mathrm{stab}}.
\]

Because \(\tilde g\) minimizes \(Q_g^{\mathrm{cal}}\),
\(
Q_{\hat g(t)}^{\mathrm{cal}}
-
Q_{\tilde g}^{\mathrm{cal}}
\ge
0.
\)
For the upper bound, write
\[
\begin{aligned}
Q_{\hat g(t)}^{\mathrm{cal}}
-
Q_{\tilde g}^{\mathrm{cal}}
=
\left\{
Q_{\hat g(t)}^{\mathrm{cal}}
-
Q_{\hat g(t)}(t)
\right\}
+
\left\{
Q_{\hat g(t)}(t)
-
Q_{\tilde g}(t)
\right\}
+
\left\{
Q_{\tilde g}(t)
-
Q_{\tilde g}^{\mathrm{cal}}
\right\}.
\end{aligned}
\]
The first and third terms are at most
\(\varepsilon_n^{\mathrm{stab}}\) in absolute value, while the middle term is
nonpositive because \(\hat g(t)\) minimizes \(Q_g(t)\). Therefore,
\[
0
\le
Q_{\hat g(t)}^{\mathrm{cal}}
-
Q_{\tilde g}^{\mathrm{cal}}
\le
2\varepsilon_n^{\mathrm{stab}}.
\]

Finally, suppose that \(\tilde g\) is the unique minimizer of
\(Q_g^{\mathrm{cal}}\) and
\[
\Delta_n^{\mathrm{cal}}
>
2\varepsilon_n^{\mathrm{stab}}.
\]
If \(\hat g(t)\ne\tilde g\) for some candidate \(t\), then the definition of
\(\Delta_n^{\mathrm{cal}}\) implies
\[
Q_{\hat g(t)}^{\mathrm{cal}}
-
Q_{\tilde g}^{\mathrm{cal}}
\ge
\Delta_n^{\mathrm{cal}}
>
2\varepsilon_n^{\mathrm{stab}},
\]
contradicting the preceding inequality. Hence
\(
\hat g(t)=\tilde g
\)
for every \(t\in\{t_0,\ldots,t_{\max}\}\). When the selected bandwidths agree,
the two procedures use the same augmented weights for every candidate and
therefore produce the same residual scores, \(p\)-values, and final stopping
rule.
\end{proof}

\end{document}